\documentclass[10pt]{article} % For LaTeX2e

\usepackage[accepted]{rlj}           % Should be uncommented for submission
\usepackage{amssymb}            % Defines common symbols like \mathbb R
\usepackage{mathtools}          % Extends amsmath, providing common math tools
\usepackage{mathrsfs}           % Enables \mathscr, which can work in cases that \mathcal does not
\usepackage{graphicx}           % For including images
\usepackage{subcaption}         % Allows for the use of subfigures and subcaptions
\usepackage[space]{grffile}     % For spaces in image names
\usepackage{url}                % For displaying URLs
\usepackage{lipsum}             % For placeholder text

\usepackage[dvipsnames]{xcolor}
\usepackage{colortbl}
\usepackage{tcolorbox}
\usepackage{adjustbox}

\usepackage{graphicx}
\usepackage{wrapfig}
\usepackage{subcaption}

\usepackage{algorithm}
\usepackage[noend]{algorithmic}
\usepackage{amsthm}
\newtheorem{theorem}{Theorem}

\usepackage{amsmath}
\usepackage{xfrac}
\usepackage{todonotes}
\usepackage{enumitem}

\usepackage{theoremref}
\usepackage{thmtools}
\usepackage{bbm}
\usepackage{thm-restate}

\newtheorem{assumption}{Assumption}

\newtheorem{remark}{Remark}
\theoremstyle{plain}
\usepackage{multicol}
\usepackage{multirow}
\usepackage{array}
\usepackage{tikz}
\usetikzlibrary{positioning, arrows.meta, shapes.geometric}
\newlength{\cellsize}
\newcolumntype{C}{>{$}c<{$}} % math-mode version of "l" column type

\newtcolorbox[auto counter]{examplebox}{
  title=Example~\thetcbcounter,
}

\usepackage{amsmath,amsfonts,bm}

\def\eqref#1{equation~\ref{#1}}
\def\1{\bm{1}}

\def\vmu{{\bm{\mu}}}
\def\vtheta{{\bm{\theta}}}
\def\vphi{{\bm{\phi}}}
\def\va{{\bm{a}}}

\def\vs{{\bm{s}}}

\DeclareMathAlphabet{\mathsfit}{\encodingdefault}{\sfdefault}{m}{sl}
\SetMathAlphabet{\mathsfit}{bold}{\encodingdefault}{\sfdefault}{bx}{n}

\def\sA{{\mathcal{A}}}

\def\sD{{\mathcal{D}}}
\def\sN{{\mathcal{N}}}

\def\sI{{\mathcal{I}}}

\def\sL{{\mathcal{L}}}

\def\sS{{\mathcal{S}}}

\DeclareMathOperator*{\argmax}{arg\,max}

\newcommand{\RL}{\textsc{rl}}

\newcommand{\KL}{\textsc{kl}}

\newcommand{\PPO}{\textsc{ppo}}

\newcommand{\pitheta}{{\pi_{\vtheta}}}

\usepackage{soul}
\usepackage{xcolor}

\usepackage[toc,page,header]{appendix}
\usepackage{minitoc}

\title{Centralized Adaptive Sampling for Reliable\\Co-Training of Independent Multi-Agent Policies}

\setrunningtitle{Centralized Adaptive Sampling for Reliable Co-Training of Independent Multi-Agent Policies}

\author{Nicholas E.~Corrado\textsuperscript{1}, Josiah P.~Hanna\textsuperscript{1}}

\emails{ncorrado@wisc.edu, jphanna@cs.wisc.edu}

\affiliations{
$^{1}$\textbf{University of Wisconsin -- Madison}
}

\contribution{
    We are the first work to show that joint sampling error can cause sub-optimal convergence in independent on-policy policy gradient RL \textit{even when the objective’s expected gradient points toward optimal behavior.}
    }
    {
    Most prior works focus on mitigating sub-optimal convergence via equilibrium selection: if optimizing the expected return yields sub-optimal equilibria, the objective is modified to steer agents toward more favorable equilibria (\textit{e.g.},~\cite{christianos2022pareto}).
    }

\contribution{
    We show that coordinated action selection can reduce joint sampling error more efficiently than independent on-policy sampling and MA-PROPS, a multi-agent extension of the PROPS sampling error correction method for single-agent RL~\citep{corrado_props_2023}.
    }
    {
    To test this hypothesis, we introduce an adaptive action sampling approach to reduce joint sampling error in the Centralized Training with Decentralized Execution setting. Our method, \textbf{Co}operative \textbf{S}ampling \textbf{E}rror \textbf{R}eduction (CoSER), continually adapts a centralized behavior policy to place higher probability on joint actions that are under-sampled w.r.t.\@ the current joint policy. CoSER is algorithmically similar to MA-PROPS but addresses a different problem: MA-PROPS targets sampling error w.r.t.\@ \textit{each agent's marginal action distribution}, whereas CoSER targets \textit{joint} sampling error. 
    }

\contribution{
    We empirically demonstrate that CoSER reduces joint sampling error more efficiently than on-policy sampling and MA-PROPS and increases the probability of independent on-policy policy gradient MARL converging optimally.
    }
    {
    With fixed agent policies, CoSER reduces joint sampling error by the same amount as on-policy sampling and MA-PROPS with 30\%--50\% fewer samples (Fig.~\ref{fig:se_joint}). In fact, MA-PROPS generally does not reduce joint sampling error. 
    During RL training, CoSER reduces joint sampling error more than on-policy sampling and MA-PROPS (Fig.~\ref{fig:rl_se}), thereby increasing the probability of converging optimally by 10\%--20\% (Fig.~\ref{fig:success_rate_all}). MA-PROPS does not improve reliability in nearly all games we study.
    }

 % Your keywords

\summary{
Independent on-policy policy gradient algorithms are widely used for multi-agent reinforcement learning (MARL) in cooperative and no-conflict games, but they are known to converge sub-optimally when each agent’s individual policy gradient points away from an optimal joint equilibrium~\citep{claus1998dynamics, lyu2021contrasting, papoudakis2020comparative, christianos2022pareto}. 
Going beyond prior work, we  observe that sub-optimal convergence can still arise \textit{even when the expected individual policy gradients of each agent point toward the optimal joint solution.} 
After collecting a finite set of trajectories, stochasticity in independent action sampling can cause the joint data distribution to deviate from the expected joint on-policy distribution. 
This \textit{sampling error} w.r.t.\@ the joint on-policy distribution  produces inaccurate gradient estimates that can make agents converge sub-optimally.
In this work, we show that joint sampling error can be reduced through coordinated action selection and that reducing joint sampling error improves the \textit{reliability} of independent on-policy policy gradient MARL (\textit{i.e.} the probability of agents converging to an optimal joint policy).
}

\begin{document}

\makeCover  % Create the cover page
\maketitle  % Make the title section

\begin{abstract}
Independent on-policy policy gradient algorithms are widely used for multi-agent reinforcement learning (MARL) in cooperative and no-conflict games, but they are known to converge sub-optimally when each agent’s individual policy gradient points away from an optimal joint equilibrium~\citep{claus1998dynamics, lyu2021contrasting, papoudakis2020comparative, christianos2022pareto}. 
%~\citep{claus1998dynamics, lyu2021contrasting, papoudakis2020comparative, christianos2022pareto}. 
%
Going beyond prior work, we observe that sub-optimal convergence can still arise \textit{even when the expected individual policy gradients of each agent point toward the optimal joint solution.} 
After collecting a finite set of trajectories, stochasticity in independent action sampling can cause the joint data distribution to deviate from the expected joint on-policy distribution. 
This \textit{sampling error} w.r.t.\@ the joint on-policy distribution  produces inaccurate gradient estimates that can make agents converge sub-optimally.
We hypothesize that joint sampling error can be reduced through coordinated action selection and that doing so will increase the \textit{reliability} of policy gradient learning in MARL (\textit{i.e.}, the probability of converging to an optimal joint policy).
To test this hypothesis, we first introduce an adaptive action sampling approach to reduce joint sampling error in the Centralized Training with Decentralized Execution setting. Our method, \textbf{Co}operative \textbf{S}ampling \textbf{E}rror \textbf{R}eduction (CoSER), continually adapts a centralized behavior policy to place higher probability on joint actions that are under-sampled w.r.t.\@ the current joint policy.
We then empirically evaluate CoSER on a diverse set of multi-agent games and demonstrate that (1) CoSER reduces joint sampling error more efficiently than independent on-policy sampling and (2) this reduction increases the reliability of independent policy gradient algorithms.
% increasing the fraction of training runs that converge to an optimal joint policy.
\end{abstract}

\section{Introduction}

On-policy policy gradient methods are among the most popular algorithms for multi-agent reinforcement learning (MARL) in cooperative and no-conflict games~\citep{de2020independent, yu2022surprising, zhou2021smarts}.\footnote{In no-conflict games, all agents share the same set of preferred equilibria. Cooperative games are a subset of no-conflict games in which all agents share the same reward function.}
A common approach is to treat agents as independent learners: each agent samples trajectories independently from its own policy (without observing the actions of other agents), estimates a Monte Carlo approximation of its policy gradient~\citep{sutton2018reinforcement}, and then performs a local policy update to maximize its expected return via gradient ascent.
%
% Despite their simplicity, 
While independent algorithms have demonstrated strong empirical performance~\citep{papoudakis2020comparative, papoudakis2020benchmarking} and powered several high-profile success stories~\citep{de2020independent, yu2022surprising, zhou2021smarts}, it is well-understood that they may not converge to the most preferred equilibrium even in no-conflict games~\citep{claus1998dynamics, lyu2021contrasting, papoudakis2020comparative, christianos2022pareto}.
In this paper, we identify a novel failure mode for independent policy gradient learning in MARL: \textit{even when the expected policy gradients of each agent align with optimal behavior}, stochasticity in action sampling can nevertheless cause agents to converge to sub-optimal solutions. We illustrate this phenomenon with the following example. 

\begin{wrapfigure}{R}{0.3\linewidth}
\vspace{-1.2em}
    \def\arraystretch{1.5}
    \centering
    % \begin{adjustbox}{0.8\linewidth}
    \begin{adjustbox}{width=\linewidth}
    \begin{tabular}{cCCC}
        \multicolumn{2}{c}{} & \multicolumn{2}{c}{Agent 2} \\
        && \multicolumn{1}{|C}{A} & B \\
        \cline{2-4}
        \multirow{2}{*}{\rotatebox{90}{Agent 1}}
        % & A &\multicolumn{1}{|C}{4,4} & 3,3 \\
        % & B &\multicolumn{1}{|C}{2,2} & 1,1 \\
                & A &\multicolumn{1}{|C}{12,12} & 0,6 \\
        & B &\multicolumn{1}{|C}{6,0} & 2,2 \\
    \end{tabular}
    \end{adjustbox}
    % \caption{$2\times2$ matrix game where $r_1,r_2$ denotes rewards for Agent 1 ($r_1$) and 2 ($r_2$).}
        \caption{Matrix game where $r_1,r_2$ denotes rewards for Agents 1 and 2, respectively.}
    \label{fig:matrix_game_example}
    \vspace{-0.5em}
\end{wrapfigure}

Consider the $2\times2$ matrix game in Fig.~\ref{fig:matrix_game_example}. Suppose each agent assigns equal probability to actions $A$ and $B$ and selects actions independently. For both agents, the expected reward of $A$ is $(12 + 0)/2 = 6$, and the expected reward of $B$ is $(6 + 2)/2 = 4$. Thus, both agents are incentivized to increase the probability of $A$, steering them toward the optimal joint action $(A, A)$.
Now suppose the agents play the game four times. In expectation, each agent will sample both actions twice and observe each joint action once. However, due to randomness in action selection, Agent 1 may sample $A, B, A, B$ while Agent 2 samples $B, A, B, A$ so that $(A,A)$ and $(B,B)$ are \textit{under-sampled} w.r.t.\@ the expected joint on-policy distribution. 
Consequently, both agents observe reward $0$ for action $A$ and reward $6$ for action $B$ and thus increase the probability of $B$, driving agents toward the sub-optimal strategy $(B, B)$.
The core issue is \textit{joint sampling error}: randomness in action sampling causes the empirical joint data distribution to deviate from the expected joint on-policy distribution, leading to inaccurate policy gradient estimates.
Moreover, joint sampling error exists even though both agents sample $A$ and $B$ twice and thus have zero sampling error w.r.t.\@ their own policies.
\begin{figure*}[t]
    \centering
    \includegraphics[width=\linewidth]{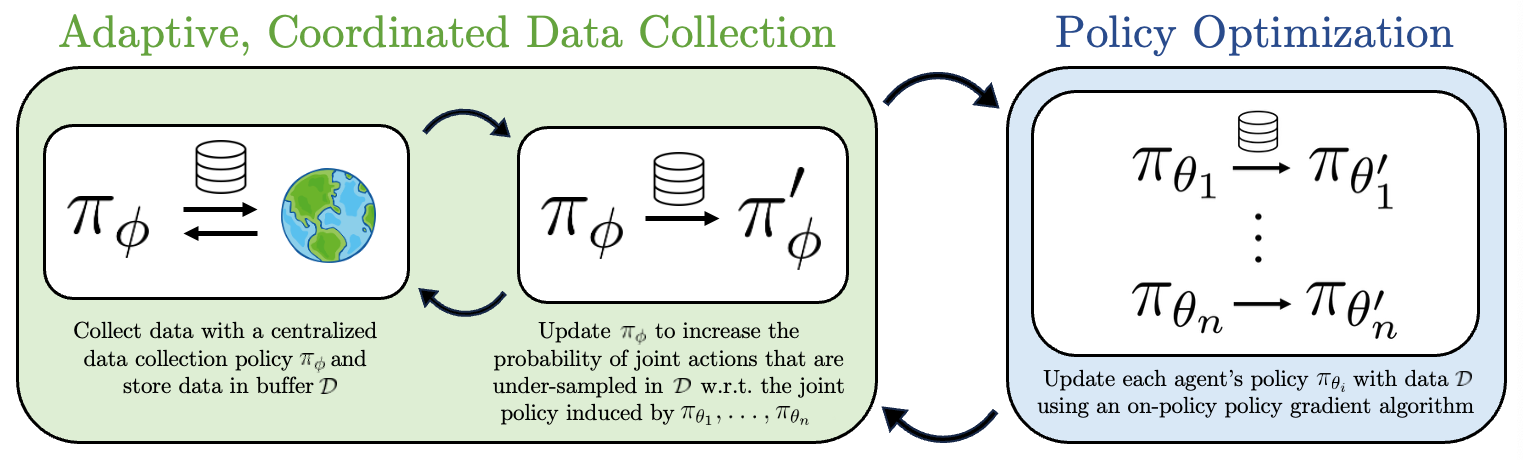}
    \caption{Rather than collecting data $\sD$ by sampling actions from each agent's current policy $\pi_{\vtheta_i}$, CoSER collects data with a centralized policy $\pi_\vphi$ that it continually adapts to reduce joint sampling error in $\sD$ with respect to the joint policy. 
    % induced by each agent's policy.
    }
    \label{fig:overview}
\vspace{-1.8em}
\end{figure*}

Under on-policy sampling, the only way to reduce sampling error is to collect more data.
Recently, \citet{corrado_props_2023} introduced an action sampling algorithm (PROPS) that reduces sampling error more efficiently than on-policy sampling. 
However, this work focused on reducing sampling error w.r.t.\@ a single policy in single-agent RL settings.
As the above example shows, reducing sampling error w.r.t.\@ each agent’s policy individually does not necessarily reduce sampling error w.r.t.\@ the joint policy.
Thus, agents must coordinate action selection to reduce joint sampling error.
These observations motivate the central question of this work: \textit{Can adaptive, coordinated action sampling reduce joint sampling error and, in doing so, increase the reliability of multi-agent policy gradient learning?} Here, reliability refers to the probability of converging to an optimal joint policy.

To answer this question, we introduce \textbf{Co}operative \textbf{S}ampling \textbf{E}rror \textbf{R}eduction (CoSER), an action sampling algorithm that adaptively corrects joint sampling error during multi-agent on-policy data collection (Fig.~\ref{fig:overview}).
Rather than sampling actions from each agent independently, CoSER samples actions from a centralized data collection policy that we continually update to increase the probability of under-sampled joint actions.
We first evaluate CoSER on a diverse set of no-conflict multi-agent games and show that it reduces joint sampling error more efficiently than standard on-policy sampling.
Next, we answer our central question affirmatively and show that reducing joint sampling error during independent on-policy policy gradient learning increases the fraction of training runs that converge to an optimal joint policy.
While centralized action sampling may be challenging in tasks with many agents due to the exponential growth of the joint action space, we view CoSER as a first step toward understanding and mitigating joint sampling error in multi-agent learning.
In summary, our contributions are:

\begin{enumerate}
    \item We identify a novel failure mode of independent on-policy policy gradient algorithms: even when the expected gradient points toward optimal behavior, joint sampling error can cause convergence to sub-optimal solutions.

    \item We show how this failure mode can be mitigated by coordinating action selection to reduce joint sampling error. To evaluate this idea empirically, we develop CoSER, a data collection algorithm that adaptively increases the probability of under-sampled joint actions.
    
    \item  We empirically demonstrate that CoSER reduces joint sampling error more efficiently than on-policy sampling and consequently makes independent on-policy policy gradient MARL converge to an optimal joint policy more reliably.
\end{enumerate}

\section{Related Work}

\textbf{Equilibrium selection.} 
MARL algorithms may converge to sub-optimal solutions when each agent’s expected gradient points away from the optimal equilibrium. 
%
% In such settings, most existing works modify the learning objective to avoid converging to sub-optimal equilibria.
%
Prior works address equilibrium selection by learning joint action-values~\citep{christianos2022pareto, littman2001friend, sunehag2017value, rashid2020monotonic, son2019qtran} or by using optimism~\citep{matignon2007hysteretic, palmer2017lenient, zhao2023optimistic}.
Most methods are off-policy with the exception of Pareto Actor-Critic~\citep{christianos2022pareto} and Optimistic Multi-Agent Policy Gradient~\citep{zhao2023optimistic}, which modify on-policy gradients to promote convergence to a Pareto-optimal equilibrium.
Prior works modify the learning objective because directly optimizing the expected return can lead to sub-optimal equilibria.
In contrast, we study settings where the expected gradients already point toward the optimal equilibrium, but joint sampling error causes convergence to sub-optimal outcomes. Rather than altering the objective or gradient update, we improve reliability by reducing joint sampling error through adaptive data collection.

\textbf{Reducing sampling error via adaptive data collection.} Prior work in single-agent RL has shown that adaptive action sampling can reduce sampling error more efficiently than standard on-policy sampling. \citet{zhong2022robust} first demonstrated this idea theoretically, and \citet{corrado_props_2023} introduced a practical and scalable algorithm (PROPS) for applying it in policy gradient learning. \citet{mukherjee2022revar} uses adaptive sampling in the bandit setting, and other bandits works~\citep{tucker2022variance, wan2022safe, konyushova2021active} use data-conditioned but non-adaptive sampling strategies. 
% With the exception of \citet{corrado_props_2023},
These methods focus on single-agent RL or policy evaluation. In contrast, we focus on multi-agent settings and highlight unique challenges posed by joint sampling error.

\textbf{Reducing sampling error via importance sampling.} Several works use importance sampling to reduce sampling error without collecting additional data for policy evaluation~\citep{precup2000eligibility}, policy gradient RL \citep{papini2018stochastic,metelli2018policy, hanna2021importance}, and temporal difference learning~\citep{ Pavse2020ReducingSE}. Similar techniques exist for contextual bandits~\citep{li2015toward, narita2019efficient}. These works reduce sampling error by reweighting previously collected data. 
In contrast, our work focuses on  whether sampling error can be controlled as data is collected.

\textbf{Coordinated Exploration.} In MARL, data collection techniques often incentivize agents to maximize state-action coverage and explore states where multi-agent interactions are likely
~\citep{mahajan2019maven, wang2019influence, liu2021cooperative, li2022pmic, gupta2021uneven, zheng2021episodic, zhang2023self}. 
% For instance, MAVEN~\citep{mahajan2019maven} uses latent variables to promote diverse joint behavior, while influence-based methods~\citep{wang2019influence, liu2021cooperative} guide agents toward states with high mutual information. 
%
While these works guide action selection to explore new joint behaviors, our work adapts action selection to more accurately approximate the on-policy distribution and improve policy gradient estimates.

\section{Preliminaries}

In this section, we formalize the multi-agent RL setting and discuss relevant multi-agent policy gradient algorithms.

\subsection{Multi-Agent Reinforcement Learning}

We model the MARL environment as a fully observable, finite-horizon stochastic game $(\sI, \sS, \sA, p, r, \gamma)$ with $n$ agents $\sI = \{1,\dots,n\}$, joint state space $\sS = \sS_1 \times \dots \times \sS_n$, and joint action space $\sA = \sA_1 \times \dots \times \sA_n$. Each agent $i$ has a discrete action space $\sA_i = \{1, \dots, k\}$ and a stochastic policy $\pi_{\vtheta_i} : \sS_i \times \sA_i \to [0,1]$ parameterized by $\vtheta_i$.
The joint policy $\pi_\vtheta = (\pi_{\vtheta_1}, \dots, \pi_{\vtheta_n})$ with parameters $\vtheta = (\vtheta_1, \dots, \vtheta_n)$ defines a distribution over joint actions conditioned on the joint state. 
For brevity, we often write $\pi_i := \pi_{\vtheta_i}$ and $\pi := \pi_\vtheta$.
The transition function $p : \sS \times \sA \times \sS \to [0,1]$ specifies the probability of transitioning to the next joint state. The reward function $r : \sS \times \sA \to \mathbb{R}^n$ returns a reward vector $(r_1, \dots, r_n)$, where $r_i$ is the reward for agent $i$.
Each agent’s expected return is
$
J_i(\vtheta) = \mathbb{E}_{\tau \sim \pi_\vtheta} \left[\sum_{t=0}^H \gamma^t r_i(\vs_t, \va_t)\right],
$
where random variable $H$ is the horizon,
and the global objective is to maximize \textit{welfare}, defined as
$
J(\vtheta) = \sum_{i \in \sI} J_i(\vtheta).
$
We focus on \textit{no-conflict games}, where all agents share the same set of optimal joint policies. We refer to the policies used for data collection as \textit{behavior policies} and the policies being optimized as 
\textit{target policies}.

\subsection{On-Policy Policy Gradient Algorithms}
On-policy policy gradient algorithms perform gradient ascent over policy parameters to maximize each agent's expected return $J_i(\vtheta)$. 
For exposition, we assume finite state and action spaces so that the policy gradient w.r.t.\@ agent $i$ can be written as:
\begin{equation}\label{eq:joint_pg}
    \nabla J_i(\vtheta) \propto \sum\nolimits_{(\vs,\va_i, \va_{-i})\in\sS\times\sA}
    d_{\pi}(\vs,\va_i, \va_{-i})\left[A_i^{\pi}(\vs,\va_i, \va_{-i}) \nabla_{\vtheta_i} \log \pi_i (\va_i|\vs_i)\right]
\end{equation}
where $\va_i$ denotes agent $i$’s action, $\va_{-i}$ denotes the actions of all agents except agent $i$, $d_{\pi}(\vs,\va_i,\va_{-i})$ denotes the joint state–action visitation distribution, and $A_i^{\pi}(\vs,\va_i,\va_{-i})$ denotes the advantage of agent $i$.
% taking joint action $(\va_i,\va_{-i})$ in state $\vs$.
%
In practice, $\nabla J_i(\vtheta)$ is approximated with Monte Carlo samples collected from $\pi$, and an estimate of $A^{\pi}$ is used in place of the true advantage~\citep{schulman2015high}.
% Many works extend single-agent policy gradient methods~\citep{williams1992simple, kakade2001natural, schulman2015trust, mnih2016asynchronous, espeholt2018impala, lillicrap2015continuous, haarnoja2018soft} to multi-agent settings. 
% Many works extend single-agent policy gradient methods~\citep{williams1992simple, schulman2015trust, mnih2016asynchronous} to multi-agent settings. 
%
Independent algorithms like IPPO~\citep{de2020independent} apply single-agent RL algorithms to each agent, treating other agents as part of the environment. 
Since independent updates may converge to sub-optimal equilibria~\citep{claus1998dynamics}, many algorithms use Centralized Training with Decentralized Execution (CTDE), 
which gives agents access to global information during training while requiring local execution. 
%
% CTDE enables extensions of single-agent policy gradient theorems~\citep{sutton1999policy, silver2014deterministic} to multi-agent settings~\citep{kuba2021trust}, and often employ centralized critics improve credit assignment~\citep{foerster2018counterfactual, kuba2021trust, yu2022surprising, ma2022value, zhong2024heterogeneous}.
%
CTDE methods typically use centralized critics to improve credit assignment~\citep{foerster2018counterfactual, kuba2021trust, yu2022surprising, ma2022value, zhong2024heterogeneous}.
We view CTDE methods like MAPPO~\citep{yu2022surprising} as independent algorithms because their centralized critics depend only on joint observations, not joint actions.

\section{Joint Sampling Error in MARL}
\label{sec:sampling_error}

We now present our primary contribution: a formal description of how joint sampling error produces inaccurate policy gradient estimates 
and how adaptive sampling can reduce joint sampling error.
%
% For exposition, we assume finite state and action spaces. The policy gradient w.r.t.\@ agent $i$ can then be written as:
%
% \begin{equation}\label{eq:policy-gradient}
% \begin{split}
%     \nabla J_i(\vtheta) &= \sum\nolimits_{(\vs,\va_i, \va_{-i})\in\sS\times\sA} \\
%     &d_{\pi}(\vs,\va_i, \va_{-i})\left[A^{\pi}(\vs,\va_i, \va_{-i}) \nabla_{\vtheta_i} \log \pi_i (\va_i|\vs_i)\right].
% \end{split}
% \end{equation}
%
Observe that the policy gradient of agent $i$ in Eq.~\ref{eq:joint_pg} is a linear combination of the gradient for each  $(\vs_i, \va_i)$ pair $\nabla_{\vtheta_i}  \log \pi_i (\va_i|\vs_i)$ weighted by $d_{\pi}(\vs,\va_i, \va_{-i})A^{\pi}(\vs,\va_i, \va_{-i})$. 
Crucially, this weighting depends on the \textit{joint} action $(\va_i, \va_{-i})$.
% Crucially, this weighting depends on the \textit{joint} visitation distribution and \textit{joint} advantage w.r.t.\@ agent $i$.
%
Let $\sD$ be a dataset of trajectories.
It is straightforward to show that the Monte Carlo estimate of the policy gradient can be written in a similar form as Eq.~\ref{eq:joint_pg} except with the true state-action visitation distribution replaced with the empirical visitation distribution $d_\sD(\vs, \va_i, \va_{-i})$, denoting the fraction of times $(\vs, \va_i, \va_{-i})$ appears in $\sD$~\citep{hanna2021importance}:
\begin{equation}\label{eq:policy-gradient}
    \nabla \widehat J_i(\vtheta) = \sum\nolimits_{(\vs,\va_i, \va_{-i})\in\sS\times\sA} d_{\sD}(\vs,\va_i, \va_{-i})\left[A_i^{\pi}(\vs,\va_i, \va_{-i}) \nabla_{\vtheta_i} \log \pi_i (\va_i|\vs_i)\right].
\end{equation}
By comparing Eq.~\ref{eq:joint_pg} and Eq.~\ref{eq:policy-gradient}, we can see how joint sampling error affects gradient estimation. When $(\vs, \va_i, \va_{-i})$ is over-sampled (\textit{i.e.}, $d_\sD(\vs,\va_i, \va_{-i}) > d_{\pi}(\vs,\va_i, \va_{-i})$), then $\nabla_{\vtheta_i} \log \pi_i(\va_i|\vs_i)$ contributes more to the overall gradient than it should. 
Similarly, when $(\vs,\va_i, \va_{-i})$ is under-sampled, $\nabla_{\vtheta_i}  \log \pi_i(\va_i|\vs_i)$ contributes less than it should.
We now provide a concrete example based on the matrix game in Fig.~\ref{fig:matrix_game_example} illustrating how small amounts of joint sampling error can cause the wrong actions to be reinforced---even when agents have access to their true advantages and have zero sampling error in the marginal visitation distribution $d_{\pi_i}(\vs,\va_i)$ of each policy individually.

% We now use the matrix game in Fig.~\ref{fig:matrix_game_example} to illustrate how small amounts of joint sampling error can cause the wrong actions to be reinforced---even if agents have access to their true advantages and zero sampling error in the marginal visitation distribution $d_{\pi_i}(\vs,\va_i)$ of each policy individually.
\begin{tcolorbox}[
                  boxsep=2pt,
                  title=\textbf{Example 1:} Joint sampling error can cause incorrect policy gradient updates,
                  left=3pt,right=3pt,
                  % boxsep=1pt,
                  % left=0pt,
                  % right=0pt,
                  % top=2pt,
                  % arc=0pt,
                  % boxrule=0pt,toprule=1pt,
                  % colback=white
                % colframe=MidnightBlue
                  ]%%

% \textbf{\underline{Example 1:} Joint sampling error can cause incorrect policy gradient updates.}
Consider two policies $\pi_1, \pi_2$ in an MDP with two discrete actions $A$ and $B$. 
Assume a direct parameterization $\pi_i(A|\vs) = \vtheta_{i, \vs}$, $\pi_i(B|\vs) = 1-\vtheta_{i, \vs}$ with $\vtheta_{i, \vs_0} = 0.5$ so that each policy places equal probability on both actions in $\vs_0$ and thus equal probability on each joint action. Then $\forall i$,
% \begin{equation*}
% \begin{split}
%     \nabla \log \pi_i(B|\vs_0) &= -\nabla \log \pi_i(A|\vs_0), \forall i \\
%     d_{\pi_i}(\vs_0, A) &= d_{\pi_i}(\vs_0, B), \forall i.
% \end{split}
% \end{equation*}
\begin{equation*}
\begin{split}
        \nabla \log \pi_i(A|\vs_0) = -\nabla \log \pi_i(B|\vs_0),  \qquad
    d_{\pi_i}(\vs_0, A) = d_{\pi_i}(\vs_0, B).
\end{split}
\end{equation*}
Suppose that in a particular state $\vs_0$, the advantages $A_i^\pi(\vs_0, \va_{i}, \va_{-i})$ w.r.t.\@ both policies are
$A_i^{\pi}(\vs_0, A, A) = 7;\;  A_i^{\pi}(\vs_0, A, B) = -5;\;
A_i^{\pi}(\vs_0, B, A) = 1;\;  A_i^{\pi}(\vs_0, B, B) = -3$, $i=1,2$.
%\footnote{These advantages correspond to the matrix game in Fig.~\ref{fig:matrix_game_example}.}
%w.r.t.\@ the joint policy $\pi((\va_i, \va_j)|\vs) = \pi_1(\va_i|\vs)\pi_2(\va_j|\vs)$ is 
% \begin{align*}
%     A^{\pi}(\vs_0, A, A) &= 7, \quad  A^{\pi}(\vs_0, A, B) = -5, \\
% A^{\pi}(\vs_0, B, A) &= 1, \quad A^{\pi}(\vs_0, B, B) = -3 
% \end{align*}
% \begin{align*}
%     A^{\pi}(\vs_0, A, A) = 7, \quad  A^{\pi}(\vs_0, A, B) = -5, \quad
% A^{\pi}(\vs_0, B, A) = 1, \quad A^{\pi}(\vs_0, B, B) = -3 
% \end{align*}
% \begin{equation*}
%     \begin{split}
%     A^{\pi}(\vs_0, A, A) &= 5  \\ % 24 - 19 = 5\\
%     A^{\pi}(\vs_0, A, B) &= -4 \\ % - 19 = -3\\
%     A^{\pi}(\vs_0, B, A) &= -3 \\ % - 19 = -1 \\
%     A^{\pi}(\vs_0, B, B) &= 2 \\ % - 19 = 2. \\
%     \end{split}
% \end{equation*}
% \begin{equation*}
%     \begin{array}{ll}
%     A^{\pi}(\vs_0, A, A) = 5 &  A^{\pi}(\vs_0, A, B) = -4 \\ 
%     A^{\pi}(\vs_0, B, A) = -3 & A^{\pi}(\vs_0, B, B) = 2 \\ 
% \end{array}
% \end{equation*}
%
Then, the expected gradient of $\pi_i$ increases the probability of sampling $A$ and thus increases the probability of observing the optimal joint action $(A, A)$:
\begin{equation}
  \sfrac{2}{4}\cdot (7 - 5) \cdot\nabla \log \pi_i(A|\vs_0) + \sfrac{2}{4}\cdot (1 - 3) \cdot\nabla \log \pi_i(B|\vs_0) =
  2\nabla \log \pi_i(A|\vs_0)  
\end{equation}
With on-policy sampling, after 4 visits to $\vs_0$, each joint action will be observed once in expectation.
However, if we actually observe $(A, B)$ 2 times and $(B, A)$ 2 times, a Monte Carlo estimate of each policy gradient yields
\begin{equation}
  \sfrac{2}{4}\cdot (-5) \cdot\nabla \log \pi_i(A|\vs_0) + \sfrac{2}{4}\cdot (-3) \cdot\nabla \log \pi_i(B|\vs_0) =
   -\nabla \log\pi_i(A|\vs_0)  
\end{equation}
which \textit{decreases} the probability of sampling the optimal $(A, A)$ action.
We emphasize that this issue arises despite having zero sampling error w.r.t.\@ the expected on-policy distribution of each policy individually (\textit{i.e.} both agents sample $A$ twice and $B$ twice.)
\end{tcolorbox}

With on-policy sampling, joint sampling error vanishes only in the limit of infinite data.
To more rapidly reduce joint sampling error, agents can instead coordinate their action selection: if a joint action is under-sampled at $\vs$, agents should increase the probability of sampling that joint action at $\vs$ in the future.
Continuing with Example 1, suppose the agents visit $\vs_0$ 4 more times.
To achieve zero sampling error, the agents should observe $(A, A)$ and $(B, B)$ twice each.
Since independent on-policy sampling gives no control over the joint distribution, they may observe $(A, B)$ and $(B, A)$ again.
If they sample their next actions from a distribution that puts probability 1 on $(A, A)$ on the first two visits to $\vs_0$ and probability 1 on $(B, B)$ on the next two, the aggregate batch of data will exactly match the expected joint distribution.

In tabular \textit{single-agent} settings, \citet{zhong2022robust} and \citet{corrado_props_2023} proved that 
selecting the most under-sampled action at each state produces an empirical state-action distribution that converges to expected visitation distribution at a faster rate than on-policy sampling.
In multi-agent settings, one might try to apply this heuristic to each agent independently.
However, our next example shows that reducing sampling error w.r.t.\@ each agent may not reduce joint sampling error.
\begin{tcolorbox}[
                  boxsep=2pt,
                  title=\textbf{Example 2:} Independent adaptive sampling may not decrease joint sampling error,
                  left=3pt,right=3pt,
                  label=example2
                  % left=0pt,
                  % right=0pt,
                  % top=2pt,
                  % arc=0pt,
                  % boxrule=0pt,toprule=1pt,
                  % colback=white
                  % colframe=MidnightBlue
                  ]%%

% \textbf{\underline{Example 2:} Independent adaptive sampling may not decrease joint sampling error.} 
Consider two policies $\pi_1$ and $\pi_2$ that put equal probability on actions $A$ and $B$ so that under on-policy sampling, each joint action is observed equally often in expectation. Now suppose each agent always selects the action most under-sampled relative to its own policy. At $t = 0$, both actions are equally under-sampled, so both agents sample from $\pi_i$. Without loss of generality, suppose agent 1 chooses $A$ and agent 2 chooses $B$. 
At $t = 1$, agent 1 selects $B$ and agent 2 selects $A$, since these actions are now under-sampled. 
This pattern repeats: the agents choose $(A, B)$ at even timesteps and $(B, A)$ at odd timesteps. 
Consequently, joint actions $(A, A)$ and $(B, B)$ are never sampled, and joint sampling error does not decrease as $t\to\infty$.
\end{tcolorbox}

% Consider two policies $\pi_1$ and $\pi_2$ that place equal probability on two discrete actions $A$ and $B$.
% %(four joint actions $(\va_i, \va_j)$, $i,j \in \{0,1\}$). 
% %
% Under on-policy sampling, both agents will in expectation sample $A$ and $B$ an equal number of times and thus observe all joint actions an equal number of times. 
% %
% Now suppose both agents always sample the action that is most under-sampled w.r.t.\@ their current policy. 
% %
% At $t=0$, both actions are under-sampled, so both agents sample their first action from $\pi_i$. 
% %
% Suppose agent 1 samples $A$ and agent 2 samples $B$.
% %
% At the next timestep $t=1$, $B$ is most under-sampled w.r.t.\@ and $A$ is most under-sample w.r.t\@ $\pi_2$, so agent 1 will sample $B$ and agent 2 will sample $A$. 
% %
% In general, for all timesteps for which $t \mod 2 \equiv 0$, the agents select joint action $(A, B)$, and for all timesteps for which $t \mod \equiv 1$,  the agents select joint action $(B, A)$.
% %
% Joint actions $(A, A)$ and $(B, B)$ will never be sampled, so joint sampling error will not tend toward zero.

This example highlights a key point: correcting joint sampling error requires agents to coordinate their action selection. 
Building upon these ideas, we now present a \textit{centralized} adaptive sampling algorithm to correct joint sampling error.
% w.r.t.\@ joint on-policy distribution.
%\todo{be consistent with "correct" vs "reduce" sampling error}

% \begin{tcolorbox}[
% % width=4in,
%                   boxsep=2pt,
%                   title=\textbf{Takeaway},
%                   % left=0pt,
%                   % right=0pt,
%                   % top=2pt,
%                   % arc=0pt,
%                   % boxrule=0pt,toprule=1pt,
%                   % colback=white
%                   ]%%

% \begin{itemize}
% Reducing sampling error w.r.t.\@ each agent's policy does not necessarily reduce sampling error w.r.t.\@ the joint policy. To reduce joint sampling error, agent's 
% \end{itemize}

% \end{tcolorbox}

% \section{Cooperative Sampling Error Reduction (CoSER)}
\section{Reducing Joint Sampling Error}

In this section, we introduce an adaptive sampling algorithm to reduce joint sampling error in multi-agent on-policy data collection. 
Algorithm~\ref{alg:on_policy_pg_adaptive} outlines our CTDE framework in which agents collect data with a centralized behavior policy $\pi_\vphi: \sS\times\sA\to [0, 1]$ to train decentralized target policies $\pi_{\vtheta_i}: \sS_i\times\sA_i\to [0, 1] $. 
First, we initialize the behavior policy to be equal to the target policy: $\pi_\vphi(\va|\vs) = \prod_{i=1}^n \pi_{\vtheta_i}(\va_i|\vs), \forall \vs$.
%
% At each step, $\pi_\vphi$ collects a transition and adds it to buffer $\sD$. 
At each step, agents act according to $\pi_\vphi$ and add the resulting transition to a buffer $\sD$.
Every $m$ steps, $\vphi$ is updated to increase the probability of under-sampled joint actions in $\sD$ w.r.t.\@ $\pi_\vtheta$. 
Every $n$ steps, each agent updates its policy parameters $\vtheta_i$ using data from $\sD$.
To keep the empirical distribution of $\sD$ close to the expected joint on-policy distribution, we continually adapt $\pi_\vphi$ to place more probability on joint actions that are under-sampled w.r.t.\@ $\pi_\vtheta$. \citet{corrado_props_2023} recently introduced an algorithm, PROPS, for making such updates in \textit{single-agent} settings. 
%
% In the remainder of this section, we provide an overview of PROPS and then discuss modifications needed to integrate it into multi-agent settings.
%
% In what follows, we first review PROPS in the single-agent setting. We then introduce a straightforward multi-agent extension, MA-PROPS, where each agent independently corrects sampling error for its own policy using a decentralized behavior policy. While natural, this extension suffers from the joint sampling error identified in Example 2 of Section~\ref{sec:sampling_error}. Finally, we introduce our proposed method, CoSER, which instead uses a centralized behavior policy to reduce joint sampling error.
In what follows, we first review PROPS and describe how its straightforward extension to the multi-agent setting fails to correct joint sampling error. Then, we introduce a new method, CoSER, that uses a centralized behavior policy to reduce joint sampling error.

% Algorithm~\ref{alg:on_policy_pg_adaptive} describes the framework we use for on-policy learning with a \textit{centralized} adaptive behavior policy. 
% % and is similar to the framework used by \citet{corrado_props_2023}.
% %
% Initially, the behavior policy $\pi_\vphi$ with parameters $\vphi$ is equal to the joint target policy $\pi_\vtheta$ in distribution.
% %
% At each step, the behavior policy collects a transition and adds it to buffer $\sD$. 
% %
% Every $m$ steps, the behavior policy updates its weights such that the next $m$ samples it collects reduces sampling error in $\sD$ with respect to the joint target policy $\pi_\vtheta$.
% %
% Every $n$ steps, each agent updates their target policy parameters $\vtheta_i$ with data from $\sD$.

% %
% To ensure that the empirical distribution of $\sD$ closely matches the expected joint on-policy distribution, we will update $\vphi$ to increase the probability of joint actions which are currently under-sampled in $\sD$ with respect to $\pi_\vtheta$.
% %
% \citet{corrado_props_2023} recently introduced an algorithm, PROPS, for making such updates in the single-agent setting.
% % \todo{Be more precise}
% %
% In the remainder of this section, we provide an overview of PROPS, describe how it reduces sampling error w.r.t.\@ a single policy, and then discuss the modifications needed to integrate it into the multi-agent setting.

\subsection{Behavior Policy Update: Proximal Robust On-Policy Sampling}

PROPS is based on a simple idea: starting at $\vphi = \vtheta$, gradient ascent on the log-likelihood $\sL_\text{LL}(\vphi) = \sum_{(\vs,\va) \in \sD} \log \pi_\vphi(\va|\vs)$ adjusts $\vphi$ to match the empirical distribution of $\sD$, increasing the probability of over-sampled actions and decreasing that of under-sampled actions. Taking a step in the opposite direction thus increases the probability of under-sampled actions.
\citet{zhong2022robust} proved that adjusting the behavior policy with a single gradient step in the direction of $-\nabla_\vphi \sL_\text{LL}(\vphi) \big|_{\vphi = \vtheta}$ at each timestep improves the rate at which the empirical data distribution converges to the expected on-policy distribution.
To enable larger behavior policy updates that more aggressively correct sampling error, \citet{corrado_props_2023} instead perform gradient ascent on a PPO-inspired clipped surrogate
\begin{equation}
\sL(\vphi, \vtheta, \vs, \va, \varepsilon) = \min\left[
-\frac{\pi_\vphi(\va|\vs)}{\pi_\vtheta(\va|\vs)},
-\mathtt{clip}\left(\frac{\pi_\vphi(\va|\vs)}{\pi_\vtheta(\va|\vs)}, 1 - \varepsilon, 1 + \varepsilon\right)
\right],
\label{eq:props}
\end{equation}
where $\mathtt{clip}$ bounds the ratio ${\pi_\vphi(\va|\vs)}/{\pi_\vtheta(\va|\vs)}$ to the interval $[1 - \varepsilon, 1 + \varepsilon]$. This objective prevents destructively large updates: it increases the probability of under-sampled actions by at most a factor of $1 + \varepsilon$, and decreases that of over-sampled actions by at most $1 - \varepsilon$. As in PPO, this clipping enables stable multi-epoch minibatch training while ensuring moderate adjustments to action probabilities.

The natural extension of PROPS to the multi-agent setting uses independent behavior policies $\pi_{\vphi_i}(\va_i|\vs_i)$ with the same parameterization as $\pi_{\vtheta_i}$, yielding a \textit{decentralized} behavior policy $\pi_\vphi(\va|\vs) = \prod_{i=1}^n \pi_{\vphi_i}(\va_i|\vs_i)$ where $\vphi = (\vphi_1, \dots,\vphi_n)$.
We call this extension Multi-Agent PROPS (MA-PROPS) and provide pseudocode in Algorithm~\ref{alg:maprops}. At each behavior update, we set $\vphi_i \gets \vtheta_i$ and perform minibatch gradient ascent on $\sL$.
Since $\pi_\vphi$ is decentralized, MA-PROPS will correct sampling error w.r.t.\@ each agent individually but may not reduce joint sampling error (see Example 2 in Section~\ref{sec:sampling_error}).
In the next section, we describe how to correct joint sampling error.

\begin{algorithm}[t]
    \caption{On-policy policy gradient algorithm with adaptive sampling} % \PROPS{} + \PPO{} Update is the same with $A(s,a) := -1$ for all $(s,a)$.
    \begin{algorithmic}[1]
        \STATE \textbf{Inputs}: Target batch size $n$, behavior update frequency $m$
        \STATE \textbf{Output:} Target policy parameters $\vtheta_1, \dots, \vtheta_n$.
        \STATE Initialize target policy parameters $\vtheta_1, \dots, \vtheta_n$, initialize behavior policy parameters $\vphi$ so that $\pi_\vphi \equiv \pi_\vtheta$, and initialize empty buffer $\sD$.
        % \STATE Initialize target policy parameters $\vtheta_1, \dots, \vtheta_n$.
        % \STATE Initialize behavior policy parameters $\vphi$ so that $\pi_\vphi \equiv \pi_\vtheta$.
        % \STATE Initialize empty buffer $\sD$.
        
        % \FOR{$t = 1, 2, \dots, T$} 
        %     \STATE Collect batch of data $\sB$ by running $\pi_\vphi$. 
        %     \STATE Append $\sB$ to buffer $\sD$.
        %     \STATE Update $\pi_\vphi$ with $\sD$ using Algorithm~\ref{alg:rosppo_update}.\label{alg:on_policy_pg_adaptive_behavior_loop_end}
        %     \ENDFOR
        %     \STATE Update $\pi_\vtheta$ with $\sD$.\label{alg:on_policy_pg_adaptive_behavior_target_update}
        \FOR{timestep $t = 1, 2, \dots$} 
            % \STATE Compute joint policy logits $\vz_\vtheta(\vs) = \log \pi_{\vtheta}(\cdot |\vs)$ 
            % \STATE Compute joint logit adjustment $\Delta_\vphi(\vs)$ 
            \STATE Collect one transition with $\pi_\vphi(\cdot|\vs)$, add it to $\sD$.
            % \STATE Collect one transition with $\pi_\vphi(\cdot|\vs, \vtheta) = \sigma(\log\pi_\vtheta(\vs) + \Delta_\vphi(\vs))$ and add it $\sD$.
            % \STATE Reset $\vphi$ such that $\Delta_\vphi(\vs) \approx \mathbf{0}, \forall \vs \in \sS$  
            % \STATE Reset behavior policy parameters $\vphi$ such that $\pi_\vphi \approx \pi_\vtheta$.
            % \STATE Update $\vphi$ with a step of gradient ascent on the PROPS loss $\frac{1}{|\sD|}\sum_{(\vs,\va) \in \sD} -\log\pi_\vphi(\va|\vs)$\\
            \IF{$t \mod m \equiv 0$}
                \STATE Update $\vphi$ with $\sD$ using   Algorithm~\ref{alg:maprops}.\label{alg:on_policy_pg_adaptive_behavior_loop_end}
            \ENDIF
            % \STATE Initialize $\vphi$ such that $\pi_\vphi \approx \pi_\vtheta$
            % \STATE Compute policy logits $\vz^{(i)}(\vs) = \log \pi_{\vtheta_i}(\cdot |\vs)$ 
            \IF{$t \mod n \equiv 0$}
                \STATE Update ${\vtheta_1}, \dots, \vtheta_n$ with $\sD$ using an on-policy policy gradient algorithm.\label{alg:on_policy_pg_adaptive_behavior_target_update}
            \ENDIF
        \ENDFOR
    \STATE \textbf{return} $\vtheta_1, \dots, \vtheta_n$
    \end{algorithmic}
    \label{alg:on_policy_pg_adaptive}
    % \vspace{-0.2em}
\end{algorithm}
% \end{minipage}
% \vspace{-1em}
% \end{wrapfigure}

% % \begin{wrapfigure}{R}{0.56\textwidth}
% % \vspace{-2.5em}
% % \begin{minipage}{0.56\textwidth}
\begin{algorithm}[t]
    \caption{Behavior policy update with CoSER or MA-PROPS}
    \begin{algorithmic}[1]
        \label{alg:maprops}
        \STATE \textbf{Inputs:} Joint target policy parameters $\vtheta = (\vtheta_1, \dots, \vtheta_n)$, buffer $\sD$, target \KL{} $\delta$, clipping coefficient $\epsilon$, \texttt{n\_epoch},  \texttt{n\_minibatch}.
        \STATE \textbf{Output:} Behavior policy parameters $\vphi$.
        % \STATE Initialize $\vphi$ such that $\Delta_\vphi(\vs) = \mathbf{0}, \forall \vs \in \sS$ 
        % \STATE Initialize $\vphi$ such that $\pi_\vphi(\cdot|\vs) \equiv \pi_\vtheta(\cdot|\vs), \forall \vs \in \sS$
        \IF{CoSER}
            \STATE Define \textit{centralized} behavior policy $\pi_\vphi(\cdot|\vs) := \sigma(\log\pi_\vtheta(\cdot |\vs) + \Delta_\vphi(\vs))$
            \STATE Initialize $\vphi$ such that $\Delta_\vphi(\vs) = \mathbf{0}, \forall \vs \in \sS$ so that $\pi_\vphi \equiv \pi_\vtheta$
        \ELSIF{MA-PROPS}
            \STATE Define \textit{decentralized} behavior policy $\pi_\vphi(\va|\vs) = \prod_{i=1}^n\pi_{\vphi_i}(\va_i|\vs_i)$. % with parameters $\vphi = (\vphi_1,\dots,\vphi_n)$
            \STATE Initialize $(\vphi_1,\dots,\vphi_n) \gets (\vtheta_1, \dots, \vtheta_n)$ so that $\pi_\vphi \equiv \pi_\vtheta$
        \ENDIF
        % \STATE $\vphi_\text{old} \gets \vphi$
        \FOR{\text{epoch} $i = 1, 2, \dots,$ \texttt{n\_epoch}}
            \FOR{minibatch $j = 1, 2, \dots,$ \texttt{n\_minibatch}}   
                \STATE Sample minibatch $\sD_j \sim \sD$ 
                % \STATE Compute the loss (Eq.~\ref{eq:props_loss})
       %          \STATE Update $\vphi$ with a step of gradient ascent on loss
       %  % $$\frac{1}{|\sD_j|}\sum_{(\vs,\va) \in \sD_j} -\log\pi_\vphi(\va|\vs, \vtheta)$$ 
       %  $$ \frac{1}{|\sD_j|}\sum_{(\vs,\va) \in \sD_j} \min\Bigg[
       % -\frac{\pi_\vphi(\va|\vs, \vtheta)}{\pi_{\vphi_\text{old}}(\va|\vs, \vtheta)},
       % -\mathtt{clip}\left(\frac{\pi_\vphi(\va|\vs, \vtheta)}{\pi_{\vphi_\text{old}}(\va|\vs, \vtheta)}, 1-\epsilon, 1+\epsilon\right)
       % \Bigg]$$
       \STATE Update $\vphi$ with a step of gradient ascent on loss $\frac{1}{|\sD_j|}\sum_{(\vs,\va)\in\sD_j} \sL(\vphi,\vtheta,\vs,\va,\varepsilon)$ (Eq.~\ref{eq:props})
        % where $\pi_\vphi(\cdot|\vs, \vtheta) = \sigma(\log\pi_\vtheta(\vs) + \Delta_\vphi(\vs))$
            \ENDFOR
        \IF{$D_{\KL{}}(\pi_\vtheta || \pi_\vphi) > \delta$}
                \STATE \textbf{return} $\vphi$
        \ENDIF
        \ENDFOR
        \STATE \textbf{return} $\vphi$        
    \end{algorithmic}
% \vspace{-1em}
\end{algorithm}

\subsection{Cooperative Sampling Error Reduction (CoSER)}
\label{sec:coser}

% We distinguish between decentralized and centralized joint behavior policies $\pi_\vphi(\va|\vs)$.
% %
% A decentralized joint beahvior policy factorizes as $\pi_\vphi(\va|\vs) = \prod_{i=1}^n \pi_{\vphi_i}(\va_i | \vs_i)$ so that each behavior policy $\pi_{\vphi_i}$ independently maps its local state $\vs_i$ to an action $\va_i$.
% Centralized $\pi_\vphi(\va|\vs)$ conditions on the full joint state and may induce arbitrary dependencies across agents’ actions. In the centralized case, $\pi_\vphi$ is a single, non-factorizable neural network with shared parameters $\vphi$.
%

% In this setting, $\pi_\vphi(\va|\vs)$ is implemented as a single, non-factorizable neural network with shared parameters $\vphi$.

To correct sampling error w.r.t.\@ the joint on-policy distribution, the behavior policy $\pi_\vphi$ must be \textit{centralized}: it must condition on the full joint state $\vs$ and allow arbitrary dependencies across agents’ actions.
% $\pi_\phi: S \times A \rightarrow [0,1] and contrast with \pi_\theta: S \times A_i \rightarrow [0,1].$
%
Moreover, since the PROPS update sets the behavior policy equal to the joint target policy at the start of each behavior update, we also require that $\pi_\vphi$ can be easily initialized to match  $\pi_\vtheta$.
% \textit{i.e.} $\pi_\vphi(\va|\vs) = \prod_{i=1}^n \pi_{\vtheta_i}(\va_i|\vs), \forall \vs$.
%
Since $\pi_\vphi$ and $\pi_\vtheta$ have different parameterizations (\textit{i.e.}, $\pi_\vtheta$ is decentralized), we cannot simply set $\vphi \gets \vtheta$.
%
% To enable this initialization, we introduce a specialized behavior policy architecture.
To enable this initialization, we introduce a specialized architecture.

We first compute the logits of the joint policy, $\log \pi_\vtheta(\va|\vs) = \sum_{i=1}^n \log \pi_{\vtheta_i}(\va_i|\vs_i)$. We then define a neural network $\Delta_\vphi(\vs): \sS \to \mathbb{R}^{|\sA|}$ that outputs an adjustment to each logit. The behavior policy logits are $\log \pi_\vtheta(\va|\vs) + \Delta_\vphi(\vs)$, and the behavior policy is 
$
\pi_\vphi(\cdot|\vs) = \sigma\left(\log \pi_\vtheta(\cdot|\vs) + \Delta_\vphi(\vs)\right),
$
where $\sigma$ denotes the softmax function.\footnote{We omit the behavior policy's dependence on $\vtheta$ for clarity, as $\vtheta$ is fixed during behavior updates.}
% \footnote{Formally, the behavior policy depends on both $\vs$ and $\vtheta$, and should be written as $\pi_\vphi(\cdot|\vs, \vtheta)$. We omit this dependence for clarity, as $\vtheta$ remains fixed during behavior updates.}
%
To ensure $\pi_\vphi$ and $\pi_\vtheta$ are equal at the start of each update, we set the final layer of $\Delta_\vphi$ to the zero vector so that $\Delta_\vphi(\vs) = \mathbf{0}$ for all $\vs \in \sS$.
Then, we perform minibatch gradient ascent on $\sL(\vphi)$ to adapt $\Delta_\vphi$, \textit{i.e.}, the logit adjustment $\Delta_\vphi(\vs)$ will increase for under-sampled joint actions and decrease for over-sampled ones.
We call this algorithm \textbf{Co}operative \textbf{S}ampling \textbf{E}rror \textbf{R}eduction (CoSER) and provide pseudocode in Algorithm~\ref{alg:maprops}.

\subsection{CoSER Convergence}

In this section, we extend the convergence analysis of \citet{zhong2022robust} from the single-agent setting to the multi-agent setting.\footnote{We include this theoretical analysis for completeness and to show how the analysis of \citet{zhong2022robust} for the single-agent setting maps to the multi-agent setting. While the second statement of Theorem~\ref{thm:maprops_kl} establishes a novel MARL result, we do not consider this analysis a core contribution of our work.}
% In this section, we build a theoretical understanding of CoSER. Our results extend the single-agent analysis of \citet{zhong2022robust} to the multi-agent setting.
% While our results largely follow from \citet{zhong2022robust}, we include them for completeness and to show how to map their results for the single-agent setting to the multi-agent setting.
% Our results extend those of \citet{zhong2022robust}, we include them for completeness and to show how to map their results for the single-agent setting to the multi-agent setting.
%
Due to space constraints, we defer all proofs to Appendix~\ref{app:theory}. 
First, we show that CoSER converges to the expected on-policy joint state visitation distribution.
Next,
we show that under CoSER, the empirical joint policy $\pi_\sD(\cdot|\vs)$ converges to the joint on-policy distribution $\pi(\cdot|\vs)$ faster than on-policy sampling. 
We also prove analogous results for each agent individually.
Our results use the following assumption:
\begin{assumption}
\label{assump:maprops_most_under_sampled}
CoSER uses a learning rate of $\alpha \to \infty$
% , clipping coefficient $\varepsilon\to\infty$, 
and the behavior policy is parameterized as a softmax function $\pi_\phi(a | s) \propto e^{\vphi_{s,a}}$, where each state-action pair ($\vs, \va$) has a parameter $\vphi_{s,a}$. This assumption implies that CoSER always selects the most under-sampled joint action in each state.
\end{assumption}
We elaborate on why Assumption~\ref{assump:maprops_most_under_sampled} implies that CoSER always selects the most under-sampled joint action in Appendix~\ref{app:theory}.
Theorem~\ref{thm:maprops_d} below shows that the empirical state visitation distributions converge to the true state visitation distributions under CoSER.
We let $d_{\sD_m}$ and $\pi_{\sD_m}$ denote the empirical joint state visitation distribution and empirical joint policy after $m$ state-action pairs have been taken, respectively.
In particular, $d_{\sD_m}(s)$ is the fraction of $m$ joint states that are $s$, and $\pi_{\sD_m}(a|s)$ is the fraction of times that joint action $a$ was observed in joint state $s$.
We use a subscript $i$ to denote analogous quantities for agent $i$.

\begin{theorem}
\label{thm:maprops_d}
% Assume that data is collected with an adaptive, centralized behavior policy that always takes the most under-sampled joint action in each state, $s$, with respect to joint policy $\pi$, i.e, $a \leftarrow \arg\max_{a'} (\pi(a'|s) - \pi_{\sD_m}(a'|s))$.
% 
Assume that $\sS$ and $\sA$ are finite.
Under CoSER with Assumption~\ref{assump:maprops_most_under_sampled}, the empirical joint state visitation distribution $d_{\sD_m}$ converges to the joint state distribution of $\pi$, $d_\pi$, with probability $1$:
%
% Let $\pi_n(a|s)$ be the empirical policy after $n$ state-action pairs have been collected and $p_n(s'|s,a)$ is the empirical state transition probability.
% Assume that $\lim_{m\rightarrow \infty} \pi_\mathcal{D}(a|s) = \pi(a|s)$
\[
    \forall s, \lim_{m\rightarrow \infty} d_{\sD_m}(s) = d_\pi(s).
\]
Moreover, the empirical state visitation distribution for each agent $i$, $d_{\sD_m,i}$, converges to the state visitation distribution of $\pi_i$, $d_{\pi_i}$, with probability $1$:
\[
    \forall s_i, \lim_{m\rightarrow \infty} d_{\sD_m,i}(s_i) = d_{\pi_i}(s_i) \quad \forall i\in \sI.
\]
% \label{proposition:adaptive}
\end{theorem}

% Our second theorem shows that joint sampling error and sampling error w.r.t.\@ each agent decrease faster under CoSER than under on-policy sampling.  
%
% In particular, no individual agent is disproportionately affected by sampling error under CoSER, which is important because high sampling error for even a single agent can lead all agents to converge to a sub-optimal joint policy.
%
In Theorem~\ref{thm:maprops_kl} below, the first result shows that CoSER decreases joint sampling error faster than on-policy sampling. We contrast this result with Example 2, which shows that MA-PROPS does not necessarily reduce joint sampling error and therefore lacks this guarantee.
The second result of this theorem establishes a novel multi-agent result: this accelerated reduction in joint sampling error also guarantees that sampling error w.r.t.\@ each agent decreases at the same accelerated rate. 
%
% This result is important because large sampling error for even a single agent can lead all agents to converge to a sub-optimal joint policy.
Because sub-optimal behavior by even a single agent (which can arise from sampling error) can cause all agents to converge to a sub-optimal joint policy~\citep{zhao2023optimistic}, reducing sampling error w.r.t.\@ each agent \textit{in addition to} reducing joint sampling error is important for reliable convergence.
\begin{theorem}
\label{thm:maprops_kl}
Let $s$ be a particular state that is visited $m$ times during data collection and assume $|\mathcal{A}| \geq 2$. Then under Assumption~\ref{assump:maprops_most_under_sampled},
\begin{enumerate}
    \item $D_{\mathrm{KL}}(\pi_{\mathcal{D}}(\cdot | s) \| \pi(\cdot | s)) = O_p\left(\sfrac{1}{m^2}\right)$ under CoSER while  
$D_{\mathrm{KL}}(\pi_{\mathcal{D}}(\cdot | s) \| \pi(\cdot | s)) = O_p\left(\sfrac{1}{m}\right)$ \\under on-policy sampling.
    \item $D_{\mathrm{KL}}(\pi_{\mathcal{D}, i}(\cdot | s) \| \pi_i(\cdot | s)) = O_p\left(\sfrac{1}{m^2}\right)$ under CoSER while  
$D_{\mathrm{KL}}(\pi_{\mathcal{D}, i}(\cdot | s) \| \pi_i(\cdot | s)) = O_p\left(\sfrac{1}{m}\right)$ \\under on-policy sampling.
\end{enumerate}
where $O_p$ denotes stochastic boundedness. 
\end{theorem}

% INITIALIZATION IS NOT THE PROBLEM. THE PROBLEM IS THE FACT THAT THE BEHAVIOR POLICY NEEDS TO BE CENTRALIZED TO CORRECT JOINT SAMPLING ERROR.

\begin{figure*}
\begin{subfigure}{0.17\linewidth}
    \centering
    \includegraphics[width=\linewidth]{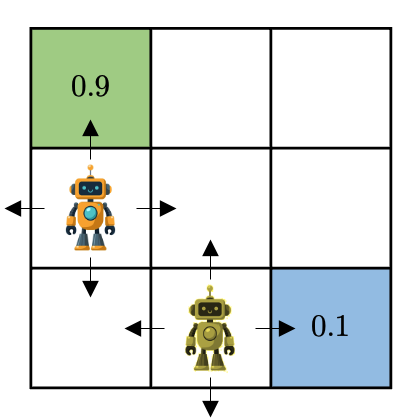}
    \caption{GridWorld}
    \label{fig:gridworld}
\end{subfigure}
\hfill
\begin{subfigure}{0.15\linewidth}
    \centering
    \raisebox{0.5em}{\includegraphics[width=\linewidth]{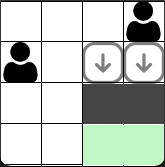}}
    \caption{BoulderPush}
    \label{fig:bpush}
\end{subfigure}
\hfill
\begin{subfigure}{0.15\linewidth}
    \centering
    \raisebox{0.5em}{\includegraphics[width=\linewidth]{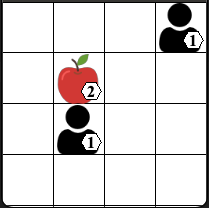}}
    \caption{LBF}
    \label{fig:lbf}
\end{subfigure}
\hfill
\begin{subfigure}{0.17\linewidth}
    \def\arraystretch{1.5}
    \centering
    \begin{adjustbox}{width=\linewidth}
    \begin{tabular}{cC|CCC}
        \multicolumn{1}{c}{} & & \multicolumn{3}{c}{Agent 2} \\
        && A & B & C \\
        \hline
        \multirow{3}{*}{\rotatebox{90}{Agent 1}}
        &A & 11^\star & -3 & 0 \\
        &B & -3 & 7^\dagger & 0 \\
        &C & 0 & 3 & 2 \\
    \end{tabular}
    \end{adjustbox}
    \caption{Climbing game}
    \label{fig:climbing}
\end{subfigure}
\hfill
\begin{subfigure}{0.17\linewidth}
    \def\arraystretch{1.5}
    \centering
    \begin{adjustbox}{width=\linewidth}
    \begin{tabular}{cC|CCC}
        \multicolumn{1}{c}{} & & \multicolumn{3}{c}{Agent 2} \\
        && A & B & C \\
        \hline
        \multirow{3}{*}{\rotatebox{90}{Agent 1}}
        &A & -7 & 0 & 10^\star\\
        &B & 0 & 2^\dagger & 0 \\
        &C & 10^\star & 0 & -7 \\
    \end{tabular}
    \end{adjustbox}
    \caption{Penalty game}
    \label{fig:penalty}
\end{subfigure}
\caption{A subset of games used in our experiments. In matrix games, $\star$ denotes optimal equilibria, and $\dagger$ denotes suboptimal equilibria. All $2\times2$  matrix games are listed in Appendix~\ref{app:games}.}
\vspace{-1em}
\end{figure*}

\section{Experiments}
\label{sec:experiments}

% \begin{wrapfigure}{R}{0.3\linewidth}
% \begin{subfigure}{\linewidth}
%     \def\arraystretch{1.5}
%     \centering
%     \begin{tabular}{cC|CCC}
%         \multicolumn{1}{c}{} & & \multicolumn{3}{c}{Agent 2} \\
%         && A & B & C \\
%         \hline
%         \multirow{3}{*}{\rotatebox{90}{Agent 1}}
%         &A & 11 & -3 & 0 \\
%         &B & -3 & 7 & 0 \\
%         &C & 0 & 3 & 2 \\
%     \end{tabular}
%     \caption{Climbing game.}
%     \label{fig:climbing2}
% \end{subfigure}\\
% \vspace{1em}
% \begin{subfigure}{\linewidth}
%     \def\arraystretch{1.5}
%     \centering
%     \begin{tabular}{cC|CCC}
%         \multicolumn{1}{c}{} & & \multicolumn{3}{c}{Agent 2} \\
%         && A & B & C \\
%         \hline
%         \multirow{3}{*}{\rotatebox{90}{Agent 1}}
%         &A & -7 & 0 & 10\\
%         &B & 0 & 2 & 0 \\
%         &C & 10 & 0 & -7 \\
%     \end{tabular}
%     \caption{Penalty game.}
%     \label{fig:penalty2}
% \end{subfigure}
% \caption{$3\times3$ matrix games.
% % We denote the Pareto optimal outcome with $*$ and the Pareto dominated outcome with $\dagger$.
% }
% \end{wrapfigure}

% \begin{figure}
%     \centering
%     \includegraphics[width=0.5\linewidth]{example-image-a}
%     \caption{Climbing game with 3 agents, 3 actions per agent.}
%     \label{fig:climbing3}
% \end{figure}

We design experiments to test the following hypotheses:

% \begin{tcolorbox}[
% % width=4in,
%                   boxsep=2pt,
%                   title=Hypotheses,
%                   % left=0pt,
%                   % right=0pt,
%                   % top=2pt,
%                   % arc=0pt,
%                   % boxrule=0pt,toprule=1pt,
%                   % colback=white
%                   ]%%
\begin{enumerate}
    \item[\textbf{H1:}] CoSER achieves lower joint sampling error than MA-PROPS and on-policy sampling after collecting a fixed number of samples.
    % \item[\textbf{H2:}] CoSER produces more accurate estimates of the joint policy gradient than Independent MA-PROPS and on-policy sampling for a fixed batch size.
    % \todo{Show for tasks with relatively small state-action spaces.}
    \item[\textbf{H2:}] CoSER yields more reliable on-policy policy gradient learning than MA-PROPS and on-policy sampling, increasing the fraction of training runs that converge optimally. 
\end{enumerate}
% Verifying these hypotheses would show that reducing joint sampling error improves the reliability of independent on-policy MARL.

% \end{tcolorbox}
% \textbf{Environments.} We use a 3x3 GridWorld environment with 2 agent one with 3 agents. The agents cannot occupy the same space, and the initial state is chosen uniformly at random.

% \textbf{Sampling error metrics.} 
\textbf{Evaluation metrics.}
To evaluate \textbf{H1}, we use two sampling error metrics. In tasks with small state-action dimensionality, we use the total variation (TV) distance between the empirical joint state-action visitation $d_\sD(\vs,\va)$ distribution, denoting the proportion of times $(\vs,\va)$ appears in buffer $\sD$, and the true joint state-action visitation distribution $d_\pitheta(\vs,\va)$ under $\pi_\vtheta$: $d_\text{TV}(d_\sD, d_{\pi_\vtheta}) =\frac{1}{2}\sum_{(\vs, \va)\in \sD} |d_\sD(\vs,\va) - d_\pitheta(\vs,\va)|$. 
In tasks where it is costly to compute $d_{\pi_\vtheta}$, we follow \citet{corrado_props_2023} and \citet{zhong2022robust} and compute sampling error as the KL divergence $D_{\KL{}}(\pi_\sD || \pi_\vtheta)$, where $\pi_\sD(\va|\vs)$is the empirical policy denoting the fraction of times $\va$ was sampled at state $\vs$ in $\sD$.
We estimate $\pi_\sD$ as the maximum likelihood policy under $\sD$ (see Appendix~\ref{app:sampling_error} for details).
To evaluate \textbf{H2}, we track the \textit{success rate}, the fraction of evaluation episodes in which the agents solve the task when sampling actions from the \textit{decentralized} target policies.
To ensure a fair comparison, we tune CoSER and MA-PROPS over the same hyperparameters and report results for the best setting (see Appendix~\ref{app:hyperparameters} for details).
In all figures, we plot the mean success rate or mean sampling error with 95\%  bootstrap confidence intervals.
% \textbf{Reliability metric.} To evaluate \textbf{H2}, we periodically evaluate agents over 100 episodes during training to track their \textit{success rate}, the fraction of evaluation episodes in which the agents solve the task. 
%
% For completeness, we additionally track the total return across all agents and provide these metrics for all experiments in Appendix~\ref{app:}.

\textbf{Multi-agent games.} We consider several multi-agent games from prior works: Level-Based Foraging (LBF)~\citep{papoudakis2020benchmarking, christianos2020shared}, BoulderPush~\citep{christianos2022pareto}, the $3\times 3$ Climbing and Penalty matrix games, and all 21 structurally distinct $2\times 2$ no-conflict matrix games (listed in Fig.~\ref{fig:2x2_games_all} of Appendix~\ref{app:games})~\citep{albrecht2024multi}. 
We additionally create a custom 3x3 GridWorld (Fig.~\ref{fig:gridworld}).
In all games, each agent has an optimal behavior, but the maximal return is achieved only when \textit{all} agents act optimally; if some do so while others do not, all agents incur a penalty.
In BoulderPush, agents must push a boulder to a goal, and both agents receive a penalty if only one agent attempts to push the boulder.
Similarly, in LBF, agents must forage a food item together, and both agents receive a penalty if only one agent attempts to forage.
In GridWorld, agents receive reward $0.9$ if they both simultaneously reach the top-left cell (green) and reward $-0.1$ if only one does. Agents receive reward $0.1$ if either agents reaches the bottom-right cell (blue). 
%
% In all games except $2\times2$ matrix games 1–16, there exists at least one suboptimal equilibrium. This equilibrium arises because the optimal action for each agent yields the optimal return only if all agents act optimally; if even one agent does not, all agents incur a penalty. As a result, agents may instead converge to an equilibrium that avoids coordination entirely, sacrificing total return to avoid the risk of being penalized.
%
In all games except $2\times2$ matrix games 1–16, there exists at least one suboptimal equilibrium in which agents avoid acting optimally to avoid the risk of being penalized.\footnote{\citet{albrecht2024multi} call this type of equilibrium \textit{risk-dominant} because it has lower risk than the optimal equilibrium. Other works refer to it as \textit{shadowed equilibrium}}
We provide additional game details in Appendix~\ref{app:games}.

\textbf{Reward rescaling.} 
% The premise of on-policy policy gradient RL is that the expected policy gradient is a desirable direction to follow. When this gradient points toward an optimal solution, reducing joint sampling error makes the gradient estimate more accurate and thus yields more reliable convergence to the optimal solution. However, if the gradient points toward a sub-optimal equilibrium, we will converge sub-optimally even if we have no joint sampling error. 
If the expected policy gradient points toward a sub-optimal equilibrium, agents will converge sub-optimally even if we have no joint sampling error.
To isolate the effect of joint sampling error on convergence, we focus on games where \textit{the expected policy gradient of each agent points toward optimality.} 
% Our work focuses on improving the reliability of independent on-policy policy gradient algorithms \textit{when the expected policy gradient of each agent points toward optimality.}
%
% In most of the games we consider, the penalties for failing to cooperate are so large that the expected policy gradient points away from optimality. 
%
However, in most of these games, the expected policy gradient points away from optimality at initialization.\footnote{\citet{christianos2022pareto} show that on-policy policy gradient algorithms like MAPPO and MAA2C consistently converge sub-optimally in the same tasks we consider.}  
% \footnote{\citet{christianos2022pareto} show that on-policy MAPPO and MAA2C converge sub-optimally in these tasks.}  
%
The probability of all agents cooperating is very small at initialization, so the expected return of any agent $i$ acting optimally (\textit{e.g.} pushing the boulder) has lower expected return than acting sub-optimally (\textit{e.g.} avoiding the boulder)~\citep{christianos2022pareto,zhao2023optimistic}.
%
% To ensure games align with our problem of interest, 
Thus, we rescale the rewards in some games so that the expected policy gradient under uniformly initialized policies encourages cooperation. 
% Without this adjustment, these game would not reflect the failure mode we aim to study and would be unsuitable for testing our hypotheses
%
We detail these modifications in Appendix~\ref{app:games}. 
%
% In brief, we reduce penalties for failed cooperation in BoulderPush, LBF, and the $3\times 3$ matrix games, and increase the reward for the optimal outcome in $2\times 2$ games 7–12 and 19–21.

% \textbf{Hyperparameter Tuning.} To ensure a fair comparison, we tune CoSER and MA-PROPS over the same hyperparameters and report results for the best setting. See Appendix~\ref{app:hyperparameters} for details.

\begin{figure*}
\begin{subfigure}{\linewidth}
    \centering
    \includegraphics[width=\linewidth]{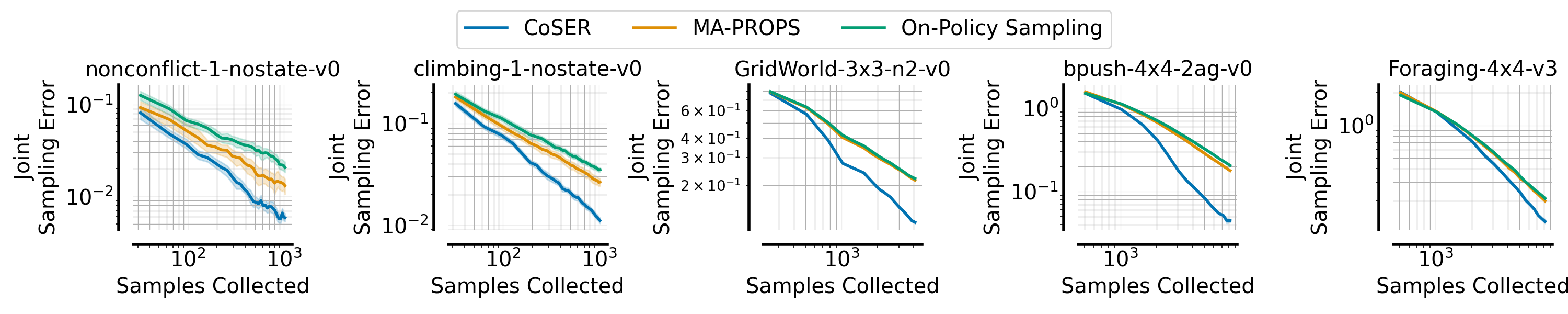}
    \caption{Sampling error w.r.t.\@ the joint policy.}
    \label{fig:se_joint}
    \vspace{1em}
    \begin{subfigure}{\linewidth}
    \centering
    \includegraphics[width=\linewidth]{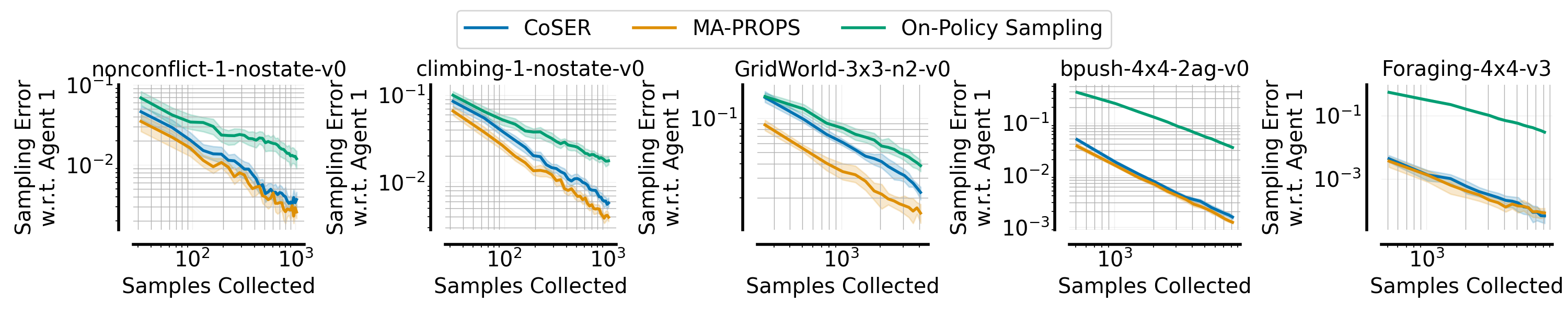}
    \caption{Sampling error w.r.t.\@ Agent 1's policy.}
    \label{fig:se_0}
    \end{subfigure}
\end{subfigure}
    \caption{
    Mean sampling error over 10 seeds with 95\% boostrap confidence intervals. {\color{MidnightBlue}\underline{Takeaway:} CoSER reduces joint sampling error with fewer samples than MA-PROPS and on-policy sampling even though MA-PROPS often reduces sampling error w.r.t.\@ Agent 1 faster than CoSER.}
    }
    \vspace{-1em}
\end{figure*}

\subsection{Reducing Sampling Error w.r.t.\@ a Fixed Target Policy}

We first study how quickly CoSER decreases sampling error when each agent's policy is a fixed. 
%
% This setting allows us to tractably and accurately estimate the true visitation distribution 
%
We provide two baselines: on-policy sampling and MA-PROPS.
We randomly initialize each agent's policy and collect a fixed number of samples with each sampling method.
%
% We use the same hyperparameter settings for CoSER and MA-PROPS and detail these in Appendix~\ref{app:}.
%
For matrix games and GridWorld, we compute sampling error as $d_\text{TV}(d_\sD, d_{\pi_{\vtheta}})$. For all other tasks, we use $D_\text{KL}(\pi_\sD||\pi_\vtheta)$.
Since all matrix games have the same dynamics, we focus on $2\times2$ game 1 and the $3\times3$ Climbing game. 
% Since matrix games are stateless, we can compute $d_{\pi_\vtheta} \equiv \pi_\vtheta$ exactly.
%
% In GridWorld, we approximate $d_\pitheta(\vs,\va)$ by collecting $10^6$ transitions sampled from the on-policy distribution.
%
% \textbf{Results.} 
As shown in Fig.~\ref{fig:se_joint}, CoSER consistently achieves lower joint sampling error than both MA-PROPS and on-policy sampling. 
%
% MA-PROPS achieves only marginally lower sampling error compared to on-policy sampling in most tasks.
MA-PROPS yields only marginal improvements over on-policy sampling in most tasks.
Fig.~\ref{fig:se_0} shows that CoSER achieves the lowest joint sampling error even though MA-PROPS decreases sampling error w.r.t.\@ each agent more than CoSER, highlighting how reducing sampling error w.r.t.\@ each agent does not guarantee reduced joint sampling error. These results support \textbf{H1}.

\subsection{Reducing Sampling Error During Reinforcement Learning}

We now examine how CoSER affects the reliability of on-policy policy gradient algorithms. We train agents using CoSER, MA-PROPS, and on-policy sampling and track their success rate and joint sampling error over training.
Our main experiments use MAPPO~\citep{yu2022surprising} to update target policies. We report additional results using IPPO~\citep{de2020independent} in Appendix~\ref{app:experiments}. 
We compute sampling error as $d_\text{TV}(d_\sD, d_{\pi_\vtheta})$ for matrix games and $D_{\KL{}}(\pi_\sD || \pi_\vtheta)$ for all other tasks.
Since CoSER and MA-PROPS generate different target policy sequences during training, we compute on-policy sampling error separately for each method by filling a second buffer with samples from the corresponding target policies.

\begin{figure*}
    \centering
    \includegraphics[width=0.95\linewidth]{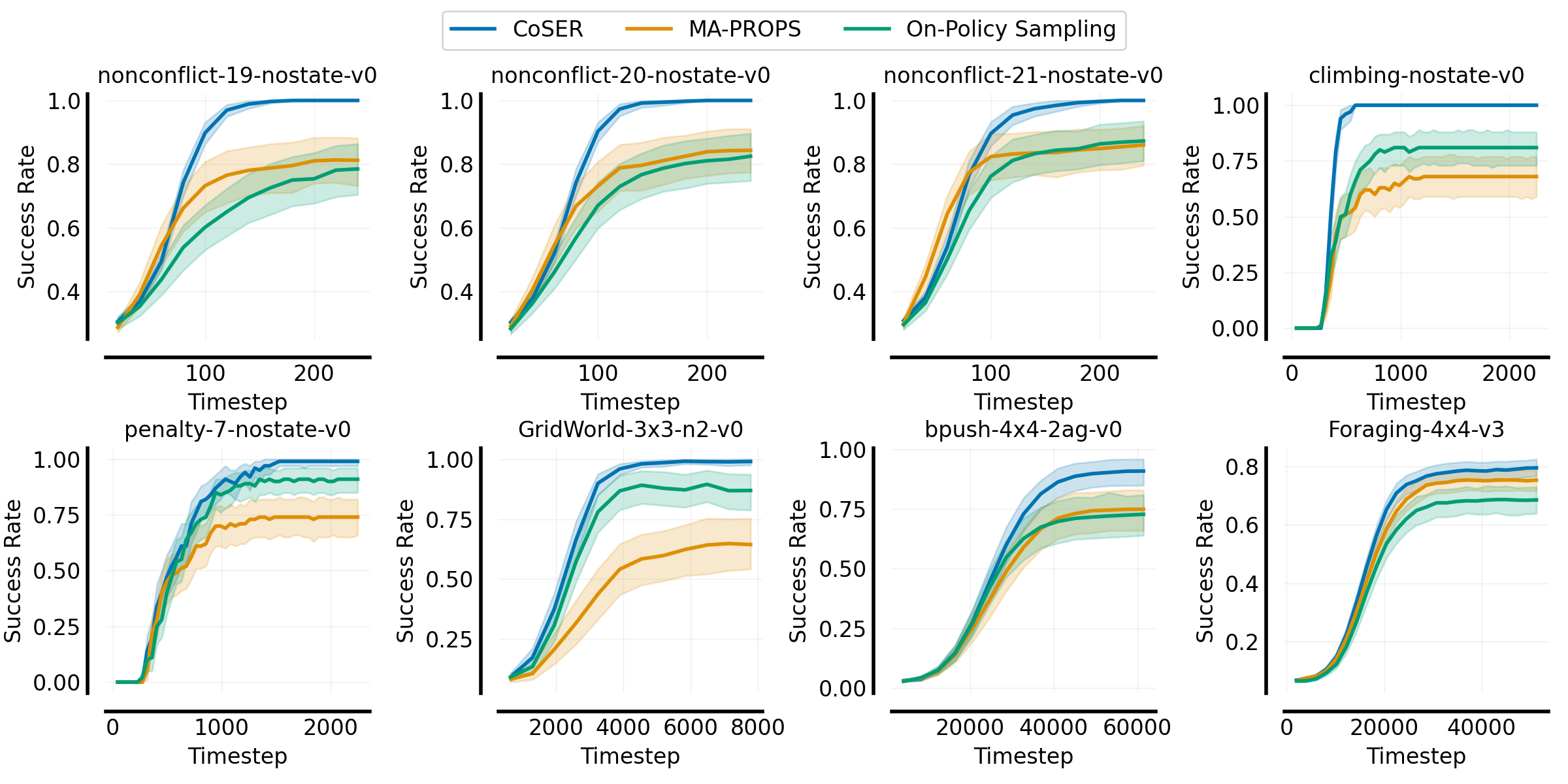}
    \caption{MAPPO mean success rates over 100 seeds with 95\% boostrap confidence intervals. {\color{MidnightBlue}\underline{Takeaway:} CoSER increases success rate by 10-20\% over on-policy sampling and MA-PROPS. MA-PROPS only increases success rate in LBF.}}
    \label{fig:success_rate_all}
    \vspace{-1em}
\end{figure*}

\textbf{GridWorld, BoulderPush, LBF.} 
As shown in Fig.~\ref{fig:success_rate_all}, CoSER has a higher probability of converging optimally in GridWorld, BoulderPush, and LBF compared to on-policy sampling and MA-PROPS, supporting \textbf{H2}.
At convergence, CoSER improves success rate over on-policy sampling by 15 percentage points in GridWorld, 19 in BoulderPush, and 10 in LBF.
In contrast, MA-PROPS provides only marginal benefit in LBF, no improvement in BoulderPush, and a 20-point drop in GridWorld.
Fig.~\ref{fig:gw_se},~\ref{fig:bpush_se}, and~\ref{fig:lbf_se} show joint sampling error during training for GridWorld, BoulderPush, and LBF and clarifies why CoSER outperforms MA-PROPS: CoSER decreases joint sampling error more than MA-PROPS. In fact, MA-PROPS does not improve over on-policy sampling at all in LBF. These results support \textbf{H1}.
%
% We provide joint sampling error curves for GridWorld in Fig.~\ref{fig:gw_3x3_joint_se} of Appendix~\ref{app:experiments}.

\textbf{Matrix Games.} We show training curves for $2\times2$ games 19-21 and both $3\times3$ games in Fig.~\ref{fig:success_rate_all}.
In these games, CoSER is the only algorithm that consistently converges to the optimal policy.
We provide training curves for the remaining $2\times2$ games in Appendix~\ref{app:experiments}.
In games 1–18, MA-PROPS and on-policy sampling achieve success rates of at least 95\% of runs. This result is consistent with findings by \citet{christianos2022pareto}, who observed that independent on-policy policy gradient algorithms like MAA2C only struggle in games 19–21. 
Nevertheless, CoSER improves reliability even in these easier games, achieving perfect success rates in all games and also converging faster in games 6, 11, 12, 16, and 17.
These results support \textbf{H2}.

Fig.~\ref{fig:2x2_21_se},~\ref{fig:climbing_se},and~\ref{fig:penalty_se} show joint sampling error curves for $2\times2$ matrix game 21 and the $3\times3$ Climbing and Penalty games. Due to space constraints, we provide joint sampling error curves the remaining $2\times 2$ matrix games in Fig.~\ref{fig:2x2_se_joint_all_ippo_maprops} of Appendix~\ref{app:experiments}.
%
% These games are easier to solve because the magnitude of the expected gradient (which points toward optimality) is large, and thus sampling error is less likely to lead to sub-optimal convergence.
%
% Due to space constraints, we provide sampling error curves for all matrix games in Appendix~\ref{app:experiments}. 
%
In all matrix games, results are qualitatively similar: CoSER decreases joint sampling error by a large margin before learning converges, while MA-PROPS generally does not decrease joint sampling error. 
Both CoSER and MA-PROPS reduce sampling error w.r.t.\@ Agent 1, again highlighting how reducing sampling error w.r.t.\@ each agent may not reduce joint sampling error.
% \footnote{Under on-policy sampling, sampling error naturally decreases throughout RL training as the agents' policies become more deterministic. Thus, it is expected that the difference in sampling error between CoSER/MA-PROPS and on-policy sampling decreases throughout training. Put differently, there is less sampling error to correct as policies become more deterministic.} 
%
These results support \textbf{H1}.

% \begin{tcolorbox}[
% % width=4in,
%                   boxsep=2pt,
%                   title=\textbf{Summary of Empirical Findings},
%                   % left=0pt,
%                   % right=0pt,
%                   % top=2pt,
%                   % arc=0pt,
%                   % boxrule=0pt,toprule=1pt,
%                   colframe=MidnightBlue
%                   ]%%

% \begin{enumerate}[leftmargin=*]
%     \item CoSER reduces joint sampling error at a faster rate than MA-PROPS and on-policy sampling even though MA-PROPS often reduces sampling error w.r.t.\@ each agent faster than CoSER.
%     \item CoSER enables MAPPO to more reliably converge to an optimal policy.
% \end{enumerate}
% % These findings show that reducing joint sampling error yields more reliable independent on-policy policy gradient
% \end{tcolorbox}

\section{Limitations}
\label{sec:limitations}

The core goal of this paper is to characterize a subtle failure mode of independent on-policy policy gradient algorithms---joint sampling error---and then demonstrate that reducing joint sampling error via adaptive action sampling can improve the reliability of these algorithms.
In this section, we discuss limitations of our proposed approach (CoSER) and outline directions for future work. 

\textbf{Scaling to many agents.}
Since the joint action space grows exponentially with the number of agents, it may be difficult to learn a centralized behavior policy for tasks with many agents. A potential solution is to use deep coordination graphs (DCG)~\citep{bohmer2020deep} to factor the behavior policy into locally conditioned components, trading off representational capacity for scalability.

\textbf{Batch size.} 
CoSER increases the probability of under-sampled joint actions at states $\vs$ in $\sD$ so that sampling error is reduced \textit{upon the agents' next visit to states with features similar to the features of $\vs$}.
If $\sD$ is small, then revisits to states with features similar to $\vs$ are rare, and CoSER will perform similarly to on-policy sampling.
%
% This limitation primarily affects tabular behavior policies, where generalization across states is limited.
%
% With function approximation, CoSER adjusts the probability of under-sampled actions at states with similar features, potentially mitigating this issue by improving generalization across the state space.
%
This could be addressed by letting $\sD$ retain historic data between target policy updates. 
\citet{corrado_props_2023} showed that PROPS can leverage historic data by collecting additional data so that the aggregate distribution in $\sD$ is close to on-policy.
Since CoSER shares the same behavior update as PROPS, it can likely benefit from the same data reuse strategies shown to be effective in single-agent settings.

\textbf{Discrete Actions.}
Our experiments focus on discrete action tasks, so the centralized behavior policy we describe in Section~\ref{sec:coser} assumes discrete actions. However, this architecture extends in principle to continuous action settings. We detail this extension in Appendix~\ref{app:continuous_action}.

\begin{figure*}
    % \vspace{-1em}
    \centering
    \begin{subfigure}{0.47\linewidth}
        \centering
        \includegraphics[width=\linewidth]{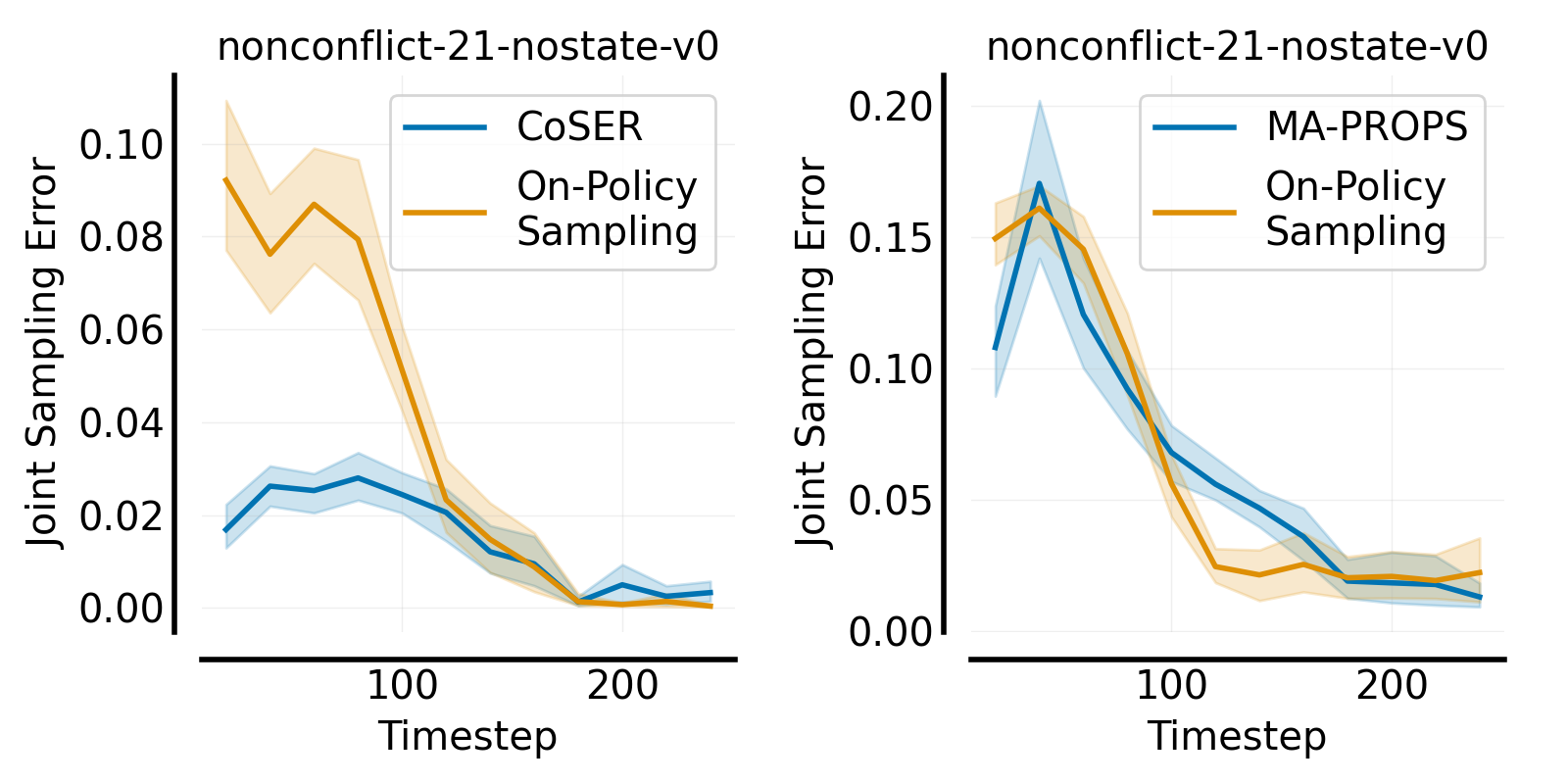}
        \caption{$2\times2$ Matrix Game 21}
        \label{fig:2x2_21_se}
    \end{subfigure}
    \begin{subfigure}{0.47\linewidth}
        \centering
        \includegraphics[width=\linewidth]{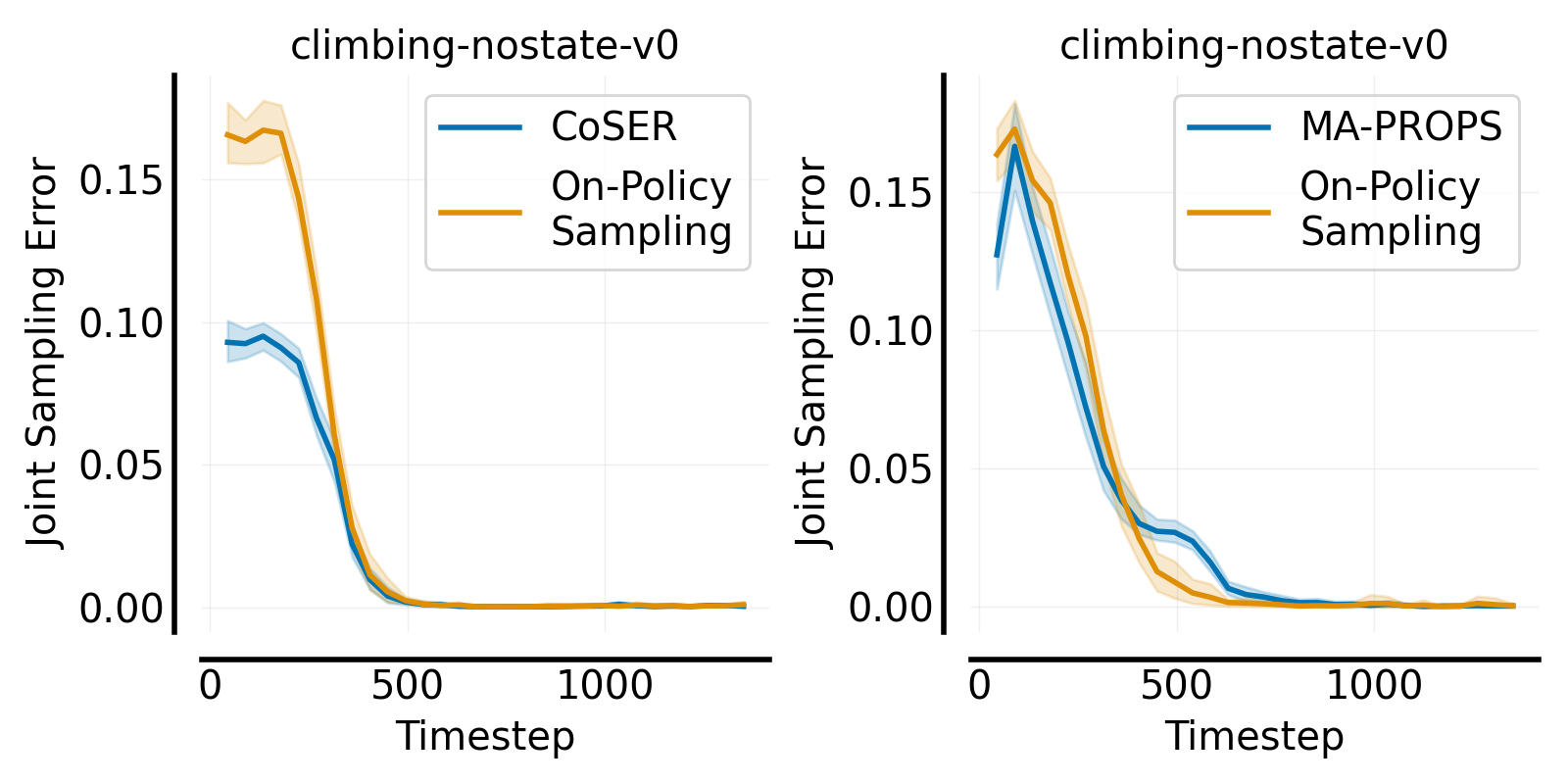}
        \caption{Climbing Game}
        \label{fig:climbing_se}
    \end{subfigure}
    \begin{subfigure}{0.47\linewidth}
        \centering
        \includegraphics[width=\linewidth]{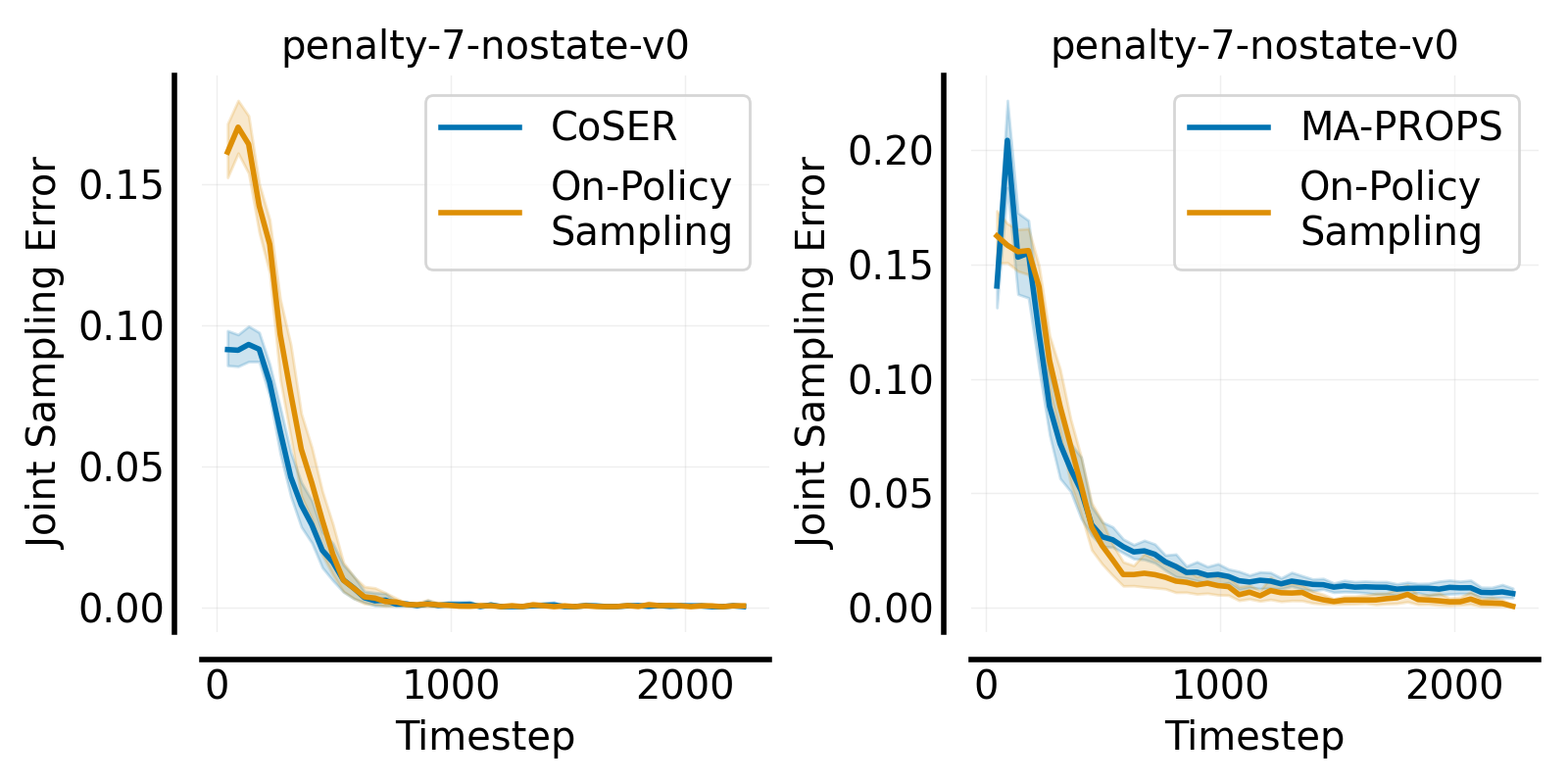}
        \caption{Penalty Game}
        % \caption{\MA-PROPS{} reduces sampling error faster than on-policy sampling and \MA-PROPS{}. The \MA-PROPS{} and on-policy sampling curves overlap. Solid curves denote means over 5 seeds. Shaded regions denote 95\% confidence intervals. }
        \label{fig:penalty_se}
    \end{subfigure}
    \begin{subfigure}{0.47\linewidth}
        \centering
        \includegraphics[width=\linewidth]{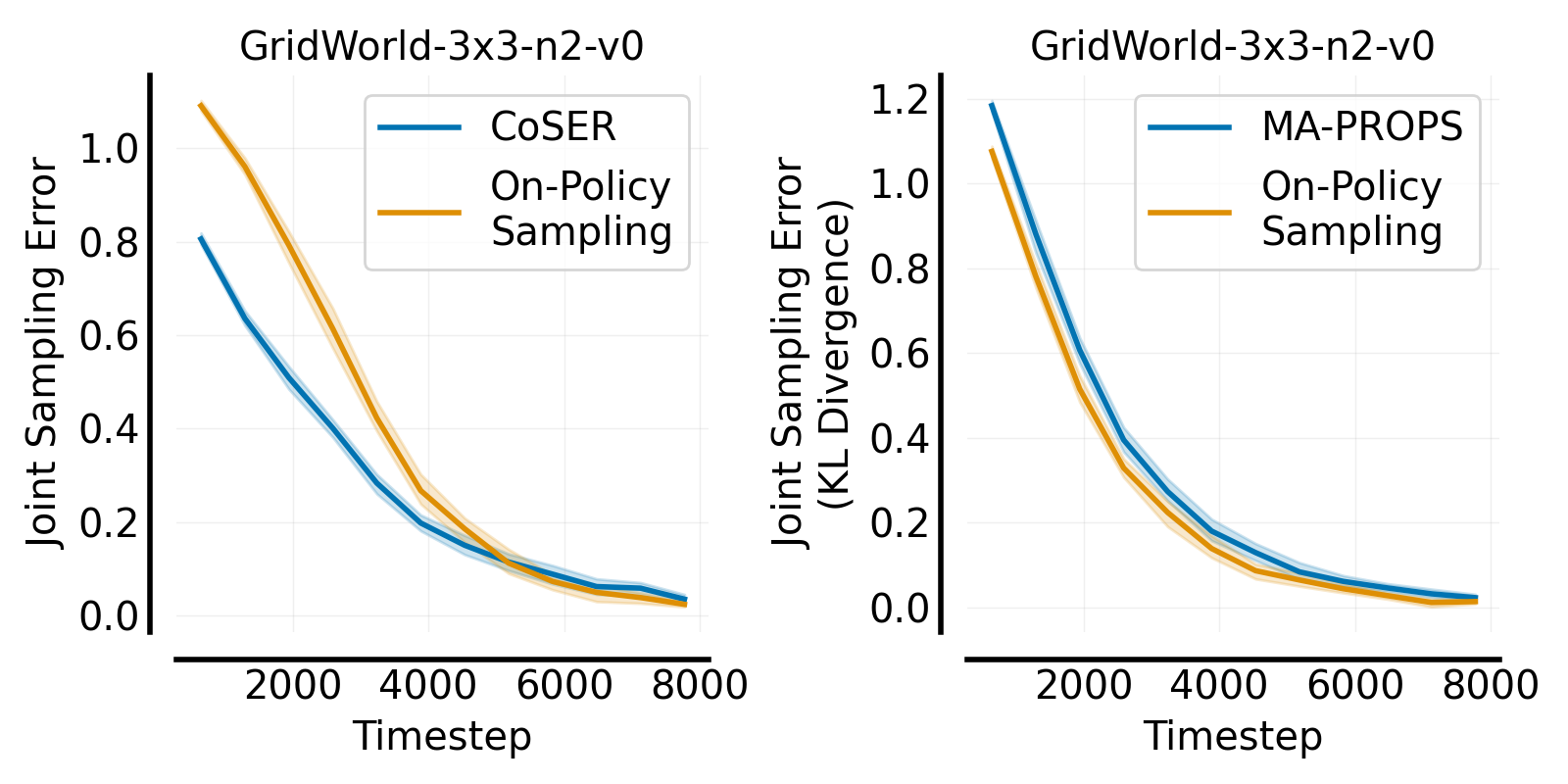}
        \caption{GridWorld}
        % \caption{\MA-PROPS{} reduces sampling error faster than on-policy sampling and \MA-PROPS{}. The \MA-PROPS{} and on-policy sampling curves overlap. Solid curves denote means over 5 seeds. Shaded regions denote 95\% confidence intervals. }
        \label{fig:gw_se}
    \end{subfigure}
    % \vspace{-1em}
    % \vspace{-1em}
    \centering
    \begin{subfigure}{0.47\linewidth}
        \centering
        \includegraphics[width=\linewidth]{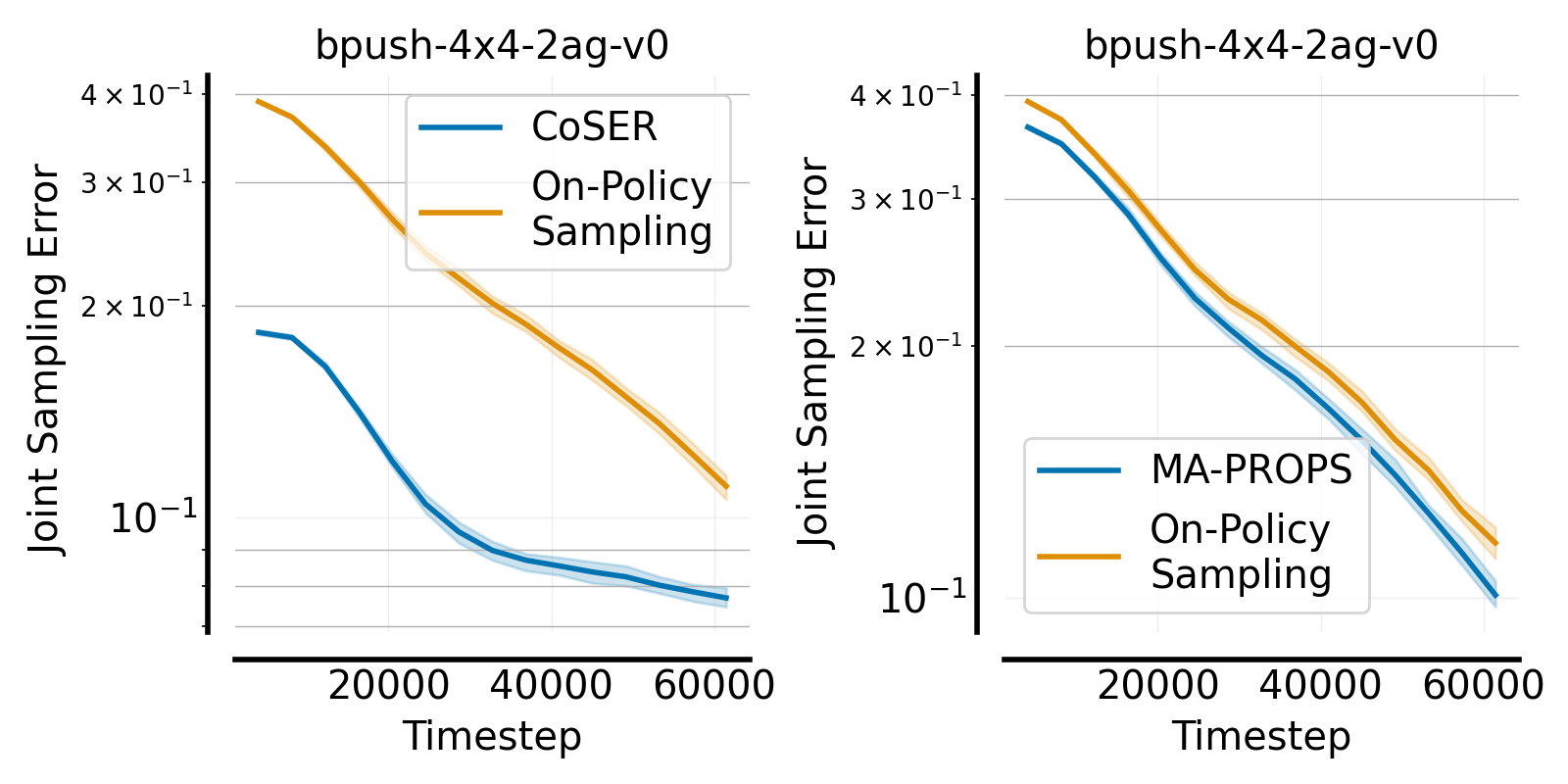}
        \caption{BoulderPush}
        % \caption{\MA-PROPS{} reduces sampling error faster than on-policy sampling and \MA-PROPS{}. The \MA-PROPS{} and on-policy sampling curves overlap. Solid curves denote means over 5 seeds. Shaded regions denote 95\% confidence intervals. }
        \label{fig:bpush_se}
    \end{subfigure}
        \begin{subfigure}{0.47\linewidth}
        \centering
        \includegraphics[width=\linewidth]{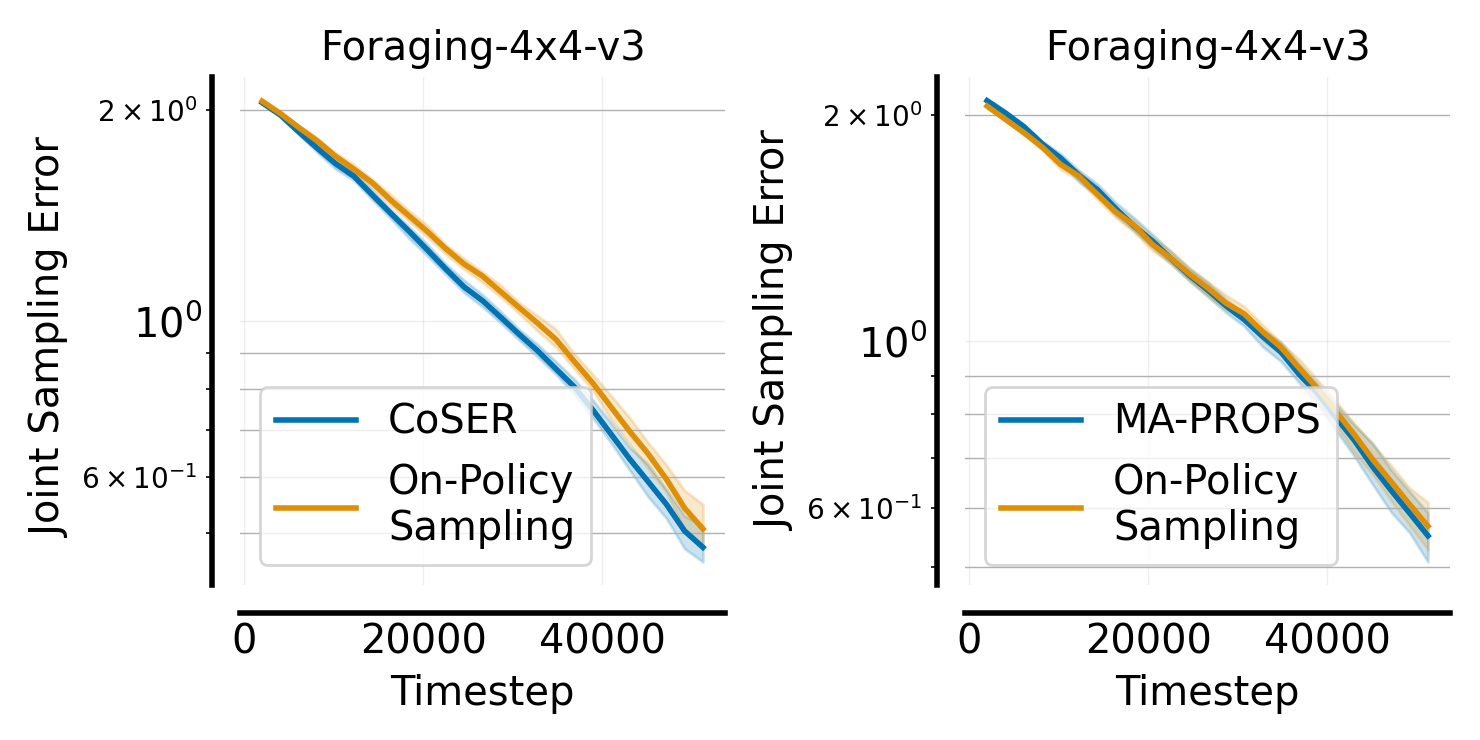}
        \caption{LBF}
        \label{fig:lbf_se}
    \end{subfigure}
    \caption{Mean joint sampling error over 100 seeds with 95\% boostrap confidence intervals. {\color{MidnightBlue}\underline{Takeaway:} CoSER reduces joint sampling error more than MA-PROPS.}}
    % \vspace{-1em}
    \label{fig:rl_se}
\end{figure*}

\section{Conclusion}

In this paper, we first identify a subtle failure mode of independent on-policy policy gradient algorithms. Stochasticity in action sampling can cause the joint distribution of collected data to deviate from the expected joint on-policy distribution, and this \textit{sampling error} can cause agents to converge to sub-optimal solutions---\textit{even when the expected policy gradients align with optimal behavior}.
Given this failure mode, we asked: Can reducing sampling error in the joint data distribution via coordinated action selection lead to more reliable convergence of independent on-policy policy gradient algorithms?
Toward answering this question, we introduce \textbf{Co}operative \textbf{S}ampling \textbf{E}rror \textbf{R}eduction (CoSER), an action sampling algorithm that adaptively corrects joint sampling error during multi-agent on-policy data collection.
CoSER follows the Centralized Training with Decentralized Execution (CTDE) paradigm:
agents collect data using a centralized behavior policy that we continually adapt to increase the probability of under-sampled joint actions, and this data is then used to train decentralized policies for deployment.
% CoSER follows the Centralized Training with Decentralized Execution (CTDE) paradigm: 
% and this data is then used to train decentralized policies.
%
Empirically, CoSER reduces sampling error in the joint on-policy distribution more efficiently than standard on-policy sampling, and increases the fraction of training runs that converge optimally.
% when using independent on-policy policy gradient algorithms.

%%%%%%%%%%%%%%%%%%%%%%%%%%%%%%%%%%%%%%%%%%%%%%%%%%%%%%%%%%%%%%%%
%% Appendices
%%%%%%%%%%%%%%%%%%%%%%%%%%%%%%%%%%%%%%%%%%%%%%%%%%%%%%%%%%%%%%%%
% \appendix

% \section{The first appendix}
% \label{sec:appendix1}
% This is an example of an appendix. 

% \noindent \textbf{Note:} Appendices appear before the references and are viewed as part of the ``main text'' and are subject to the 8--12 page limit, are peer reviewed, and can contain content central to the claims of the paper. 

% \section{The second appendix}
% \label{sec:appendix2}
% This is an example of a second appendix. If there is only a single section in the appendix, you may simply call it ``Appendix'' as follows:

% \section*{Appendix}
% % No label, since this can't be referenced meaningfully with \ref{}.
% This format should only be used if there is a single appendix (unlike in this document).

% \label{sec:ack}
% Use unnumbered third level headings for the acknowledgments. All acknowledgments, including those to funding agencies, go at the end of the paper. Only add this information once your submission is accepted and deanonymized. The acknowledgments do not count towards the 8--12 page limit.

%%%%%%%%%%%%%%%%%%%%%%%%%%%%%%%%%%%%%%%%%%%%%%%%%%%%%%%%%%%%%%%%
%% NOTE: THIS MARKS THE END OF THE "MAIN TEXT"
%%%%%%%%%%%%%%%%%%%%%%%%%%%%%%%%%%%%%%%%%%%%%%%%%%%%%%%%%%%%%%%%

%%%%%%%%%%%%%%%%%%%%%%%%%%%%%%%%%%%%%%%%%%%%%%%%%%%%%%%%%%%%%%%%
%% Bibliography
%%%%%%%%%%%%%%%%%%%%%%%%%%%%%%%%%%%%%%%%%%%%%%%%%%%%%%%%%%%%%%%%
\bibliography{main}

@article{corrado_props_2023,
  title={On-policy policy gradient reinforcement learning without on-policy sampling},
  author={Corrado, Nicholas E and Hanna, Josiah P},
  journal={arXiv preprint arXiv:2311.08290},
  year={2023}
}

@article{de2020independent,
  title={Is independent learning all you need in the starcraft multi-agent challenge?},
  author={De Witt, Christian Schroeder and Gupta, Tarun and Makoviichuk, Denys and Makoviychuk, Viktor and Torr, Philip HS and Sun, Mingfei and Whiteson, Shimon},
  journal={arXiv preprint arXiv:2011.09533},
  year={2020}
}

@inproceedings{foerster2018counterfactual,
  title={Counterfactual multi-agent policy gradients},
  author={Foerster, Jakob and Farquhar, Gregory and Afouras, Triantafyllos and Nardelli, Nantas and Whiteson, Shimon},
  booktitle={Proceedings of the AAAI conference on artificial intelligence},
  volume={32},
  number={1},
  year={2018}
}

@inproceedings{ma2022value,
  title={Value-decomposition multi-agent proximal policy optimization},
  author={Ma, Yanhao and Luo, Jie},
  booktitle={2022 China Automation Congress (CAC)},
  pages={3460--3464},
  year={2022},
  organization={IEEE}
}

@article{zhong2024heterogeneous,
  title={Heterogeneous-agent reinforcement learning},
  author={Zhong, Yifan and Kuba, Jakub Grudzien and Feng, Xidong and Hu, Siyi and Ji, Jiaming and Yang, Yaodong},
  journal={Journal of Machine Learning Research},
  volume={25},
  number={32},
  pages={1--67},
  year={2024}
}

@article{kuba2021trust,
  title={Trust region policy optimisation in multi-agent reinforcement learning},
  author={Kuba, Jakub Grudzien and Chen, Ruiqing and Wen, Muning and Wen, Ying and Sun, Fanglei and Wang, Jun and Yang, Yaodong},
  journal={arXiv preprint arXiv:2109.11251},
  year={2021}
}

@article{yu2022surprising,
  title={The surprising effectiveness of ppo in cooperative multi-agent games},
  author={Yu, Chao and Velu, Akash and Vinitsky, Eugene and Gao, Jiaxuan and Wang, Yu and Bayen, Alexandre and Wu, Yi},
  journal={Advances in Neural Information Processing Systems},
  volume={35},
  pages={24611--24624},
  year={2022}
}

@book{sutton2018reinforcement,
  title={Reinforcement learning: An introduction},
  author={Sutton, Richard S and Barto, Andrew G},
  year={2018},
  publisher={MIT press}
}

@article{precup2000eligibility,
  title={Eligibility traces for off-policy policy evaluation},
  author={Precup, Doina},
  journal={Computer Science Department Faculty Publication Series},
  pages={80},
  year={2000}
}

@inproceedings{kigma2015adam,
  author    = {Diederik P. Kingma and
               Jimmy Ba},
  editor    = {Yoshua Bengio and
               Yann LeCun},
  title     = {Adam: {A} Method for Stochastic Optimization},
  booktitle = {3rd International Conference on Learning Representations, {ICLR} 2015,
               San Diego, CA, USA, May 7-9, 2015, Conference Track Proceedings},
  year      = {2015},
  url       = {http://arxiv.org/abs/1412.6980},
  bibsource = {dblp computer science bibliography, https://dblp.org}
}

@article{hanna2021importance,
  title={Importance sampling in reinforcement learning with an estimated behavior policy},
  author={Hanna, Josiah P and Niekum, Scott and Stone, Peter},
  journal={Machine Learning},
  volume={110},
  number={6},
  pages={1267--1317},
  year={2021},
  publisher={Springer}
}

@inproceedings{Pavse2020ReducingSE,
  title={Reducing Sampling Error in Batch Temporal Difference Learning},
  author={Brahma S. Pavse and Ishan Durugkar and Josiah P. Hanna and Peter Stone},
  booktitle={International Conference on Machine Learning},
  year={2020}
}

@article{konyushova2021active,
  title={Active offline policy selection},
  author={Konyushova, Ksenia and Chen, Yutian and Paine, Thomas and Gulcehre, Caglar and Paduraru, Cosmin and Mankowitz, Daniel J and Denil, Misha and de Freitas, Nando},
  journal={Advances in Neural Information Processing Systems},
  volume={34},
  pages={24631--24644},
  year={2021}
}

@article{tucker2022variance,
  title={Variance-optimal augmentation logging for counterfactual evaluation in contextual bandits},
  author={Tucker, Aaron David and Joachims, Thorsten},
  journal={arXiv preprint arXiv:2202.01721},
  year={2022}
}

@inproceedings{wan2022safe,
  title={Safe exploration for efficient policy evaluation and comparison},
  author={Wan, Runzhe and Kveton, Branislav and Song, Rui},
  booktitle={International Conference on Machine Learning},
  pages={22491--22511},
  year={2022},
  organization={PMLR}
}

@inproceedings{narita2019efficient,
  title={Efficient counterfactual learning from bandit feedback},
  author={Narita, Yusuke and Yasui, Shota and Yata, Kohei},
  booktitle={Proceedings of the AAAI Conference on Artificial Intelligence},
  volume={33},
  number={01},
  pages={4634--4641},
  year={2019}
}

@inproceedings{li2015toward,
  title={Toward minimax off-policy value estimation},
  author={Li, Lihong and Munos, R{\'e}mi and Szepesv{\'a}ri, Csaba},
  booktitle={Artificial Intelligence and Statistics},
  pages={608--616},
  year={2015},
  organization={PMLR}
}

@inproceedings{mukherjee2022revar,
  title={ReVar: Strengthening policy evaluation via reduced variance sampling},
  author={Mukherjee, Subhojyoti and Hanna, Josiah P and Nowak, Robert D},
  booktitle={Uncertainty in Artificial Intelligence},
  pages={1413--1422},
  year={2022},
  organization={PMLR}
}

@article{zhong2022robust,
  title={Robust On-Policy Sampling for Data-Efficient Policy Evaluation in Reinforcement Learning},
  author={Zhong, Rujie and Zhang, Duohan and Sch{\"a}fer, Lukas and Albrecht, Stefano and Hanna, Josiah},
  journal={Advances in Neural Information Processing Systems},
  volume={35},
  pages={37376--37388},
  year={2022}
}

@article{huang2022cleanrl,
  author  = {Shengyi Huang and Rousslan Fernand Julien Dossa and Chang Ye and Jeff Braga and Dipam Chakraborty and Kinal Mehta and João G.M. Araújo},
  title   = {CleanRL: High-quality Single-file Implementations of Deep Reinforcement Learning Algorithms},
  journal = {Journal of Machine Learning Research},
  year    = {2022},
  volume  = {23},
  number  = {274},
  pages   = {1--18},
  url     = {http://jmlr.org/papers/v23/21-1342.html}
}

@inproceedings{schulman2015high,
  title={High-dimensional continuous control using generalized advantage estimation},
  author={Schulman, John and Moritz, Philipp and Levine, Sergey and Jordan, Michael and Abbeel, Pieter},
  booktitle={International Conference on Learning Representations (ICLR)},
  year={2016}
}

@article{claus1998dynamics,
  title={The dynamics of reinforcement learning in cooperative multiagent systems},
  author={Claus, Caroline and Boutilier, Craig},
  journal={AAAI/IAAI},
  volume={1998},
  number={746-752},
  pages={2},
  year={1998}
}

@article{christianos2022pareto,
  title={Pareto actor-critic for equilibrium selection in multi-agent reinforcement learning},
  author={Christianos, Filippos and Papoudakis, Georgios and Albrecht, Stefano V},
  journal={arXiv preprint arXiv:2209.14344},
  year={2022}
}

@article{papoudakis2020benchmarking,
  title={Benchmarking multi-agent deep reinforcement learning algorithms in cooperative tasks},
  author={Papoudakis, Georgios and Christianos, Filippos and Sch{\"a}fer, Lukas and Albrecht, Stefano V},
  journal={arXiv preprint arXiv:2006.07869},
  year={2020}
}

@book{albrecht2024multi,
  title={Multi-agent reinforcement learning: Foundations and modern approaches},
  author={Albrecht, Stefano V and Christianos, Filippos and Sch{\"a}fer, Lukas},
  year={2024},
  publisher={MIT Press}
}

@article{christianos2020shared,
  title={Shared experience actor-critic for multi-agent reinforcement learning},
  author={Christianos, Filippos and Sch{\"a}fer, Lukas and Albrecht, Stefano},
  journal={Advances in neural information processing systems},
  volume={33},
  pages={10707--10717},
  year={2020}
}

@article{zhang2023self,
  title={Self-motivated multi-agent exploration},
  author={Zhang, Shaowei and Cao, Jiahan and Yuan, Lei and Yu, Yang and Zhan, De-Chuan},
  journal={arXiv preprint arXiv:2301.02083},
  year={2023}
}

@inproceedings{gupta2021uneven,
  title={Uneven: Universal value exploration for multi-agent reinforcement learning},
  author={Gupta, Tarun and Mahajan, Anuj and Peng, Bei and B{\"o}hmer, Wendelin and Whiteson, Shimon},
  booktitle={International Conference on Machine Learning},
  pages={3930--3941},
  year={2021},
  organization={PMLR}
}

@inproceedings{liu2021cooperative,
  title={Cooperative exploration for multi-agent deep reinforcement learning},
  author={Liu, Iou-Jen and Jain, Unnat and Yeh, Raymond A and Schwing, Alexander},
  booktitle={International conference on machine learning},
  pages={6826--6836},
  year={2021},
  organization={PMLR}
}

@article{wang2019influence,
  title={Influence-based multi-agent exploration},
  author={Wang, Tonghan and Wang, Jianhao and Wu, Yi and Zhang, Chongjie},
  journal={arXiv preprint arXiv:1910.05512},
  year={2019}
}

@article{zheng2021episodic,
  title={Episodic multi-agent reinforcement learning with curiosity-driven exploration},
  author={Zheng, Lulu and Chen, Jiarui and Wang, Jianhao and He, Jiamin and Hu, Yujing and Chen, Yingfeng and Fan, Changjie and Gao, Yang and Zhang, Chongjie},
  journal={Advances in Neural Information Processing Systems},
  volume={34},
  pages={3757--3769},
  year={2021}
}

@article{mahajan2019maven,
  title={Maven: Multi-agent variational exploration},
  author={Mahajan, Anuj and Rashid, Tabish and Samvelyan, Mikayel and Whiteson, Shimon},
  journal={Advances in neural information processing systems},
  volume={32},
  year={2019}
}

@article{lyu2021contrasting,
  title={Contrasting centralized and decentralized critics in multi-agent reinforcement learning},
  author={Lyu, Xueguang and Xiao, Yuchen and Daley, Brett and Amato, Christopher},
  journal={arXiv preprint arXiv:2102.04402},
  year={2021}
}

@article{papoudakis2020comparative,
  title={Comparative evaluation of cooperative multi-agent deep reinforcement learning algorithms},
  author={Papoudakis, Georgios and Christianos, Filippos and Sch{\"a}fer, Lukas and Albrecht, Stefano V},
  journal={arXiv preprint arXiv:2006.07869},
  year={2020}
}

@inproceedings{zhou2021smarts,
  title={Smarts: An open-source scalable multi-agent rl training school for autonomous driving},
  author={Zhou, Ming and Luo, Jun and Villella, Julian and Yang, Yaodong and Rusu, David and Miao, Jiayu and Zhang, Weinan and Alban, Montgomery and Fadakar, Iman and Chen, Zheng and others},
  booktitle={Conference on robot learning},
  pages={264--285},
  year={2021},
  organization={PMLR}
}

@inproceedings{bohmer2020deep,
  title={Deep coordination graphs},
  author={B{\"o}hmer, Wendelin and Kurin, Vitaly and Whiteson, Shimon},
  booktitle={International Conference on Machine Learning},
  pages={980--991},
  year={2020},
  organization={PMLR}
}

@inproceedings{papini2018stochastic,
  title={Stochastic variance-reduced policy gradient},
  author={Papini, Matteo and Binaghi, Damiano and Canonaco, Giuseppe and Pirotta, Matteo and Restelli, Marcello},
  booktitle={International conference on machine learning},
  pages={4026--4035},
  year={2018},
  organization={PMLR}
}

@article{metelli2018policy,
  title={Policy optimization via importance sampling},
  author={Metelli, Alberto Maria and Papini, Matteo and Faccio, Francesco and Restelli, Marcello},
  journal={Advances in Neural Information Processing Systems},
  volume={31},
  year={2018}
}

@article{zhao2023optimistic,
  title={Optimistic multi-agent policy gradient},
  author={Zhao, Wenshuai and Zhao, Yi and Li, Zhiyuan and Kannala, Juho and Pajarinen, Joni},
  journal={arXiv preprint arXiv:2311.01953},
  year={2023}
}

@article{li2022pmic,
  title={Pmic: Improving multi-agent reinforcement learning with progressive mutual information collaboration},
  author={Li, Pengyi and Tang, Hongyao and Yang, Tianpei and Hao, Xiaotian and Sang, Tong and Zheng, Yan and Hao, Jianye and Taylor, Matthew E and Tao, Wenyuan and Wang, Zhen and others},
  journal={arXiv preprint arXiv:2203.08553},
  year={2022}
}

@inproceedings{littman2001friend,
  title={Friend-or-foe Q-learning in general-sum games},
  author={Littman, Michael L and others},
  booktitle={ICML},
  volume={1},
  number={2001},
  pages={322--328},
  year={2001}
}

@article{sunehag2017value,
  title={Value-decomposition networks for cooperative multi-agent learning},
  author={Sunehag, Peter and Lever, Guy and Gruslys, Audrunas and Czarnecki, Wojciech Marian and Zambaldi, Vinicius and Jaderberg, Max and Lanctot, Marc and Sonnerat, Nicolas and Leibo, Joel Z and Tuyls, Karl and others},
  journal={arXiv preprint arXiv:1706.05296},
  year={2017}
}

@article{rashid2020monotonic,
  title={Monotonic value function factorisation for deep multi-agent reinforcement learning},
  author={Rashid, Tabish and Samvelyan, Mikayel and De Witt, Christian Schroeder and Farquhar, Gregory and Foerster, Jakob and Whiteson, Shimon},
  journal={Journal of Machine Learning Research},
  volume={21},
  number={178},
  pages={1--51},
  year={2020}
}

@inproceedings{son2019qtran,
  title={Qtran: Learning to factorize with transformation for cooperative multi-agent reinforcement learning},
  author={Son, Kyunghwan and Kim, Daewoo and Kang, Wan Ju and Hostallero, David Earl and Yi, Yung},
  booktitle={International conference on machine learning},
  pages={5887--5896},
  year={2019},
  organization={PMLR}
}

@article{palmer2017lenient,
  title={Lenient multi-agent deep reinforcement learning},
  author={Palmer, Gregory and Tuyls, Karl and Bloembergen, Daan and Savani, Rahul},
  journal={arXiv preprint arXiv:1707.04402},
  year={2017}
}

@inproceedings{matignon2007hysteretic,
  title={Hysteretic q-learning: an algorithm for decentralized reinforcement learning in cooperative multi-agent teams},
  author={Matignon, La{\"e}titia and Laurent, Guillaume J and Le Fort-Piat, Nadine},
  booktitle={2007 IEEE/RSJ International Conference on Intelligent Robots and Systems},
  pages={64--69},
  year={2007},
  organization={IEEE}
}
\bibliographystyle{rlj}

%%%%%%%%%%%%%%%%%%%%%%%%%%%%%%%%%%%%%%%%%%%%%%%%%%%%%%%%%%%%%%%%
% AUTHOR: If your paper has no supplementary materials, you may 
%         comment out the line below, which creates the title for
%         the supplementary materials.
%%%%%%%%%%%%%%%%%%%%%%%%%%%%%%%%%%%%%%%%%%%%%%%%%%%%%%%%%%%%%%%%
\beginSupplementaryMaterials

\section{Convergence Analysis}
\label{app:theory}

In this section, we present the proof of Theorems~\ref{thm:maprops_d} and~\ref{thm:maprops_kl} from the main paper. 
The proofs rest on Assumption 2 and Theorem 2 by \citet{zhong2022robust} as well as Proposition 1 by \citet{corrado_props_2023}, which we first restate below for completeness.
Recall that ROS is the  algorithm introduced by~\citet{zhong2022robust} to reduce sampling error w.r.t.\@ a single policy in the single-agent setting.
\begin{assumption}[Assumption 2 in~\citet{zhong2022robust}]
ROS uses a step-size of $\alpha \to \infty$
% , clipping coefficient $\varepsilon\to\infty$, 
and the behavior policy is parameterized as a softmax function, i.e., $\pi_\phi(a | s) \propto e^{\vphi_{s,a}}$, where for each state $\vs$ and action $\va$, we have a parameter $\vphi_{s,a}$. This assumption implies that ROS always selects the most under-sampled joint action in each state.
\label{assump:ros_most_under_sampled}
\end{assumption}
\begin{theorem}[Theorem 1 in~\citet{zhong2022robust}]
\label{thm:ros_kl}
Let $s$ be a particular state that is visited $m$ times during data collection and assume that $|\mathcal{A}| \geq 2$. Under Assumption~\ref{assump:ros_most_under_sampled}, 
$
D_{\mathrm{KL}}(\pi_{\mathcal{D}}(\cdot | s) \| \pi(\cdot | s)) = O_p\left(\frac{1}{m^2}\right)$ under ROS sampling while  
$D_{\mathrm{KL}}(\pi_{\mathcal{D}}(\cdot | s) \| \pi(\cdot | s)) = O_p\left(\frac{1}{m}\right)$ under on-policy sampling
where $O_p$ denotes stochastic boundedness.
\end{theorem}
\begin{theorem}[Proposition 1 in~\citet{corrado_props_2023}]
\label{prop:props_d}
Let $\vs$ be a state that we visit $m$ times. Under Assumption~\ref{assump:ros_most_under_sampled}, we have $\forall a \in \mathcal{A}$ that:
\[
\lim_{m \to \infty} \pi_{\mathcal{D}}(a|s) = \pi(a|s).
\]
\end{theorem}
\medskip
We now prove similar results for CoSER. We first make an assumption similar to Assumption~\ref{assump:ros_most_under_sampled} posed by \citet{zhong2022robust}.

\begin{assumption}[Restated Assumption~\ref{assump:maprops_most_under_sampled}]
\label{assump:maprops_most_under_sampled_restated}
CoSER uses a learning rate of $\alpha \to \infty$
% , clipping coefficient $\varepsilon\to\infty$, 
and the behavior policy is parameterized as a softmax function, i.e., $\pi_\phi(a | s) \propto e^{\vphi_{s,a}}$, where for each state $\vs$ and action $\va$, we have a parameter $\vphi_{s,a}$. This assumption implies that CoSER always selects the most under-sampled joint action in each state.
\end{assumption}

To see why this assumption implies that CoSER always selects the most under-sampled joint action in each state, recall that at the start of every behavior policy update, we have $\pi_\vphi \equiv \pi_\vtheta$ so that $\pi_\vphi(a|s)/\pi_\vtheta(a|s) = 1$ for all $(s,a)$. Thus, clipping is not applied, and the CoSER gradient reduces to the ROS gradient:
\begin{equation}
\begin{split}
    \frac{1}{|\sD|}\sum_{(\vs,\va)\in\sD}\nabla_\vphi\sL(\vphi, \vtheta, \vs, \va, \varepsilon) &= \frac{1}{|\sD|}\sum_{(\vs,\va)\in\sD} \nabla_\vphi \left(-\frac{\pi_\vphi(\va|\vs)}{\pi_\vtheta(\va|\vs)}\right)\\
    &= \frac{1}{|\sD|}\sum_{(\vs,\va)\in\sD} - \nabla_\vphi \log\pi_\vphi(\va|\vs)\frac{\pi_\vphi(\va|\vs)}{\pi_\vtheta(\va|\vs)} \\
    &= \frac{1}{|\sD|}\sum_{(\vs,\va)\in\sD} - \nabla_\vphi \log\pi_\vphi(\va|\vs)
\end{split}
\end{equation}
Thus, since Assumption~\ref{assump:ros_most_under_sampled} implies that ROS samples the most under-sampled action, so does CoSER.

Our first theorem shows how  empirical state visitation distributions converge to their true state visitation distributions under CoSER.  
We use $d_{\sD_m}$ and $\pi_{\sD_m}$, as the empirical state visitation distribution and empirical policy after $m$ state-action pairs have been taken, respectively.
That is, $d_{\sD_m}(s)$ is the proportion of the $m$ states that are $s$, $\pi_{\sD_m}(a|s)$ is the proportion of the time that action $a$ was observed in state $s$.
We use subscript $i$ to denote analogous quantities for agent $i$.

\begin{theorem}[Restated Theorem~\ref{thm:maprops_d}]
\label{thm:maprops_d_restated}
% Assume that data is collected with an adaptive, centralized behavior policy that always takes the most under-sampled joint action in each state, $s$, with respect to joint policy $\pi$, i.e, $a \leftarrow \arg\max_{a'} (\pi(a'|s) - \pi_{\sD_m}(a'|s))$.
% 
% We further assume that $\sS$ and $\sA$ are finite.
% 
Assume that $\sS$ and $\sA$ are finite.
Under CoSER with Assumption~\ref{assump:maprops_most_under_sampled_restated},
the empirical joint state visitation distribution, $d_{\sD_m}$, converges to the joint state distribution of $\pi$, $d_\pi$, with probability $1$:
\[
    \forall s, \lim_{m\rightarrow \infty} d_{\sD_m}(s) = d_\pi(s).
\]
Moreover, the empirical state visitation distribution for each agent $i$, $d_{\sD_m,i}$, converges to the state distribution of $\pi_i$, $d_{\pi_i}$, with probability $1$:
\[
    \forall s_i, \lim_{m\rightarrow \infty} d_{\sD_m,i}(s_i) = d_{\pi_i}(s_i) \quad \forall i\in \sI.
\]
\end{theorem}
\begin{proof}
    Since the behavior policy $\pi(a|s) = \prod_{i=1}^n\pi_i(a|s_i)$ is a single agent mapping joint states to joint actions, we can  immediately apply Theorem~\ref{prop:props_d} to obtain the result for the joint state visitation distribution. The result for each agent's state visitation distribution follows from marginalizing the joint state visitation distribution: $$\lim_{m \to \infty} d_{\sD_m}(s_i) = \lim_{m \to \infty} \sum_{s_{-i}} d_{\sD_m}(s)
= \sum_{s_{-i}} \lim_{m \to \infty} d_{\sD_m}(s)
= \sum_{s_{-i}} d_\pi(s)
= d_{\pi_i}(s_i)$$
\end{proof}

% The next theorem shows that joint sampling error and sampling error w.r.t. each agent decrease faster under CoSER than under on-policy sampling. 
Our second theorem shows that joint sampling error decreases faster under CoSER than under on-policy sampling. 
We further establish a novel multi-agent result: this accelerated reduction in joint sampling error also guarantees that sampling error w.r.t.\@ each agent decreases at the same accelerated rate. 
We let $\pi_i$ denote the policy of agent $i$ and let $\pi_{\sD,i}$ denote the empirical policy of agent $i$.
\begin{theorem}[Restated Theorem~\ref{thm:maprops_kl}]
\label{thm:maprops_kl_restated}
Let $s$ be a particular state that is visited $m$ times during data collection and assume that $|\mathcal{A}| \geq 2$. Under Assumption~\ref{assump:maprops_most_under_sampled}, we have
\begin{enumerate}
    \item $D_{\mathrm{KL}}(\pi_{\mathcal{D}}(\cdot | s) \| \pi(\cdot | s)) = O_p\left(\frac{1}{m^2}\right)$ under CoSER while  
$D_{\mathrm{KL}}(\pi_{\mathcal{D}}(\cdot | s) \| \pi(\cdot | s)) = O_p\left(\frac{1}{m}\right)$ under on-policy sampling
    \item $D_{\mathrm{KL}}(\pi_{\mathcal{D}, i}(\cdot | s) \| \pi_i(\cdot | s)) = O_p\left(\frac{1}{m^2}\right)$ under CoSER while  
$D_{\mathrm{KL}}(\pi_{\mathcal{D}, i}(\cdot | s) \| \pi_i(\cdot | s)) = O_p\left(\frac{1}{m}\right)$ under on-policy sampling
\end{enumerate}
where $O_p$ denotes stochastic boundedness. 
\end{theorem}
\begin{proof}
Since the behavior policy $\pi(a|s) = \prod_{i=1}^n\pi_i(a|s_i)$ is a single agent mapping joint states to joint actions, we can immediately  apply the convergence result from Theorem~\ref{thm:ros_kl} to obtain the convergence result for joint sampling error follows.
    Since $D_{\mathrm{KL}}(\pi_{\mathcal{D}}(\cdot | s) \| \pi(\cdot | s)) \geq D_{\mathrm{KL}}(\pi_{\mathcal{D}, i}(\cdot | s) \| \pi_i(\cdot | s))$ for all agents $i\in \sI$, the result follows the result for joint sampling error.
\end{proof}
\begin{remark}
    This result implies that no individual agent is disproportionately affected by sampling error under CoSER.
    % , which is important because high sampling error for even a single agent can lead all agents to converge to a sub-optimal joint policy.
    Because sub-optimal behavior by even a single agent (which can arise from sampling error) can cause all agents to converge to a sub-optimal joint policy~\citep{zhao2023optimistic}, reducing sampling error w.r.t.\@ each agent is important for reliable convergence.
\end{remark}

\newpage
\section{Computing Sampling Error}
\label{app:sampling_error}

% We claim that \PROPS{} improves the data efficiency of on-policy learning by reducing sampling error in the agent's buffer $\sD$ with respect to the agent's current (target) policy.
%%
Similar to \citet{zhong2022robust} and \citet{corrado_props_2023}, we measure sampling error as the \KL{} divergence $D_{\KL{}}(\pi_\sD || \pi_\vtheta)$ between the empirical joint policy $\pi_\sD$ and the target joint policy $\pi_\vtheta$:
\begin{equation}
    D_{\KL{}}(\pi_\sD || \pi_\vtheta) 
    = \mathbb{E}_{\vs\sim \sD, \va\sim\pi_\sD(\cdot|\vs)}\left[\log\left(\frac{\pi_\sD(\va|\vs)}{\pi_\vtheta(\va|\vs)}\right)\right].
    % \approx \sum_{(\vs,\va)\in\sD}\log\left(\frac{\pi_\sD(\va|\vs)}{\pi_\vtheta(\va|\vs)}\right) 
\end{equation}
We compute a parametric estimate of $\pi_\sD$ by letting $\vtheta'$ be the parameters of neural network that takes joint states as input and outputs a distribution over joint actions $\pi_{\vtheta'}(\cdot |\vs)$ and maximizing the log-likelihood of $\sD$ under $\pi_\vtheta'$
\begin{equation}
    \vtheta_\textsc{MLE} = \argmax_{\vtheta'}\sum_{(\vs,\va)\in\sD}\log\pi_{\vtheta'}(\va|\vs)
\end{equation}
using stochastic gradient ascent.
After computing $\vtheta_\textsc{MLE}$, we then estimate sampling error using the Monte Carlo estimator:
\begin{equation}
    D_{\KL{}}(\pi_\sD || \pi_\vtheta) 
    \approx \sum_{(\vs,\va)\in\sD}
    \left(\log\pi_{\vtheta_\textsc{MLE}}(\va|\vs) - \log\pi_\vtheta(\va|\vs)\right).
\end{equation}
We compute sampling error w.r.t.\@ individual agents in a similar fashion: We compute estimate $\pi_{\sD_i}$ as the maximum likelihood policy under $\sD_i$ and then estimate sampling error using the Monte Carlo estimator:
\begin{equation}
    D_{\KL{}}(\pi_{\sD_i} || \pi_{\vtheta_i}) 
    \approx \sum_{(\vs_i,\va_i)\in\sD}
    \left(\log\pi_{\vtheta_{i,\textsc{MLE}}}(\va_i|\vs_i) - \log\pi_{\vtheta_i}(\va_i|\vs_i)\right).
\end{equation}

\section{Extending CoSER to Continuous Action Tasks}
\label{app:continuous_action}
The centralized behavior policy we use is specific to softmax policies for discrete action settings. 
However, this architecture easily extends to continuous action settings. 
Continuous policies are typically parameterized as Gaussians $\sN(\vmu_i(\vs_i),\sigma_i(\vs_i))$ with mean $\vmu_i(\vs_i)$ and variance $\sigma_i(\vs_i)$.
The joint policy is then a Gaussian with mean $\vmu(\vs) = (\vmu_1(\vs_1), \dots, \vmu_n(\vs_n))$ and variance $\sigma(\vs) = (\sigma_1(\vs_1), \dots, \sigma_n(\vs_n))$.
Similar to how the behavior network add an adjustment to the joint logits in the discrete setting, in a continuous setting, the behavior network would output an adjustment to the joint mean and joint variance.

More concretely,  we first compute the mean $\vmu(\vs)$ and variance $\sigma(\vs) $ of the joint policy.
Next, we define a neural network $\Delta_{\vphi}(\vs): \sS \to \mathbb{R}^{|\sA|}$ that outputs an adjustment to the mean and variance of each joint action dimension. Let $\Delta^\mu_{\vphi}(\vs)$ denote the mean adjustments and let $\Delta^\sigma_{\vphi}(\vs)$ denote the variance adjustments. The behavior policy is then\footnote{We again omit the behavior policy's dependence on $\vtheta$ for clarity, as $\vtheta$ is fixed during behavior updates.}
\begin{equation}
    \pi_\vphi(\cdot|\vs) = \sN\left(\mu(\vs) + \Delta^\mu_\vphi(\vs),  \sigma(\vs) + \Delta^\sigma_\vphi(\vs)\right)
\end{equation}
To initialize behavior policy to match $\pi_\vtheta$ at the start of each update, we set the final layer of $\Delta_\vphi$ to the zero vector so that $\Delta^\mu_\vphi(\vs) = \mathbf{0}$ and $\Delta^\sigma_\vphi(\vs) = \mathbf{0}$ for all $\vs \in \sS$.
Then, we perform minibatch gradient ascent on $\sL(\vphi)$ to adapt the logit adjustment so that the behavior policy places larger probability on under-sampled actions just as described in Algorithm~\ref{alg:maprops}.

% At the start of each behavior policy update, the centralized behavior policy should be equal to the joint policy.
% %
% One could in principle achieve this initialization via behavior cloning, but this would be too costly to run as a subroutine in each behavior update.
% %
% We instead initialize the behavior policy equal to the joint pol

\begin{figure}[H]
\begin{subfigure}{0.49\linewidth}
    \def\arraystretch{1.5}
    \centering
    \begin{tabular}{cC|CCC}
        \multicolumn{1}{c}{} & & \multicolumn{3}{c}{Agent 2} \\
        && A & B & C \\
        \hline
        \multirow{3}{*}{\rotatebox{90}{Agent 1}}
        &A & 11^\star & -30 & 0 \\
        &B & -30 & 7^\dagger & 0 \\
        &C & 0 & 6 & 5 \\
    \end{tabular}

    \caption{Original Climbing game~\citep{claus1998dynamics}.}
    \label{fig:climbing_og}
\end{subfigure}
\hfill
\begin{subfigure}{0.49\linewidth}
    \def\arraystretch{1.5}
    \centering
    \begin{tabular}{cC|CCC}
        \multicolumn{1}{c}{} & & \multicolumn{3}{c}{Agent 2} \\
        && A & B & C \\
        \hline
        \multirow{3}{*}{\rotatebox{90}{Agent 1}}
        &A & 11^\star & -3 & 0 \\
        &B & -3 & 7^\dagger & 0 \\
        &C & 0 & 3 & 2 \\
    \end{tabular}

    \caption{Modified Climbing game.}
    \label{fig:climbing2}
\end{subfigure}
\end{figure}

\begin{figure}[H]
\begin{subfigure}{0.49\linewidth}
    \def\arraystretch{1.5}
    \centering
    \begin{tabular}{cC|CCC}
        \multicolumn{1}{c}{} & & \multicolumn{3}{c}{Agent 2} \\
        && A & B & C \\
        \hline
        \multirow{3}{*}{\rotatebox{90}{Agent 1}}
        &A & -k & 0 & 10^\star\\
        &B & 0 & 2^\dagger & 0 \\
        &C & 10^\star & 0 & -k \\
    \end{tabular}
    \caption{Original Penalty game~\citep{claus1998dynamics}.}
    \label{fig:penalty_og}
\end{subfigure}
\hfill
\begin{subfigure}{0.49\linewidth}
    \def\arraystretch{1.5}
    \centering
    \begin{tabular}{cC|CCC}
        \multicolumn{1}{c}{} & & \multicolumn{3}{c}{Agent 2} \\
        && A & B & C \\
        \hline
        \multirow{3}{*}{\rotatebox{90}{Agent 1}}
        &A & -7 & 0 & 10^\star\\
        &B & 0 & 2^\dagger & 0 \\
        &C & 10^\star & 0 & -7 \\
    \end{tabular}
    \caption{Modified Penalty game. We choose $k = 7$.}
    \label{fig:penalty2}
\end{subfigure}
\caption{Original and modified $3\times3$ matrix games, where $\star$ denotes optimal equilibria, and $\dagger$ denotes suboptimal equilibria.
% We denote the Pareto optimal outcome with $*$ and the Pareto dominated outcome with $\dagger$.
}
\end{figure}

\section{Multi-Agent Games}
\label{app:games}

In this appendix, we further describe each multi-agent game we use. We additionally detail how we modify some games to ensure the expected policy gradient encourages cooperation at initialization.

\subsection{Game Descriptions}

\begin{figure}
    \centering
    \def\arraystretch{1.3}
    \newcommand{\game}[5]{
        \small
        \begin{tabular}{cCCC}
            \multicolumn{2}{c}{} & \multicolumn{2}{c}{Agent 2} \\
            && \multicolumn{1}{|C}{A} & B \\
            \cline{2-4}
            \multirow{2}{*}{\rotatebox{90}{Agent 1}} 
                & A &\multicolumn{1}{|C}{\bf{#1}} & #2 \\
                & B &\multicolumn{1}{|C}{#3} & #4 \\
        \end{tabular}
        \caption*{\small Game #5}
    }

    % Row 1
    \begin{subfigure}{0.3\textwidth}
        \centering
        \game{4,4}{3,3}{2,2}{1,1}{1}
    \end{subfigure}
    \begin{subfigure}{0.3\textwidth}
        \centering
        \game{4,4}{3,3}{2,1}{2,1}{2}
    \end{subfigure}
    \begin{subfigure}{0.3\textwidth}
        \centering
        \game{4,4}{3,2}{2,3}{1,1}{3}
    \end{subfigure}
    \vspace{1em}

    % Row 2
    \begin{subfigure}{0.3\textwidth}
        \centering
        \game{4,4}{3,2}{3,2}{1,1}{4}
    \end{subfigure}
    \begin{subfigure}{0.3\textwidth}
        \centering
        \game{4,4}{3,1}{2,1}{1,3}{5}
    \end{subfigure}
    \begin{subfigure}{0.3\textwidth}
        \centering
        \game{4,4}{3,3}{2,1}{1,2}{6}
    \end{subfigure}
    \vspace{1em}

    % Row 3
    \begin{subfigure}{0.3\textwidth}
        \centering
        \game{4,5}{3,2}{1,1}{1,1}{7}
    \end{subfigure}
    \begin{subfigure}{0.3\textwidth}
        \centering
        \game{4,5}{3,2}{1,1}{2,3}{8}
    \end{subfigure}
    \begin{subfigure}{0.3\textwidth}
        \centering
        \game{4,5}{3,2}{2,3}{1,1}{9}
    \end{subfigure}
    \vspace{1em}

    % Row 4
    \begin{subfigure}{0.3\textwidth}
        \centering
        \game{4,5}{3,1}{1,1}{2,2}{10}
    \end{subfigure}
    \begin{subfigure}{0.3\textwidth}
        \centering
        \game{4,5}{3,1}{1,1}{2,3}{11}
    \end{subfigure}
    \begin{subfigure}{0.3\textwidth}
        \centering
        \game{4,5}{3,1}{2,3}{2,1}{12}
    \end{subfigure}
    \vspace{1em}

    % Row 5
    \begin{subfigure}{0.3\textwidth}
        \centering
        \game{4,4}{2,3}{3,1}{1,3}{13}
    \end{subfigure}
    \begin{subfigure}{0.3\textwidth}
        \centering
        \game{4,4}{2,3}{3,1}{2,2}{14}
    \end{subfigure}
    \begin{subfigure}{0.3\textwidth}
        \centering
        \game{4,4}{2,2}{3,1}{1,3}{15}
    \end{subfigure}
    \vspace{1em}

    % Row 6
    \begin{subfigure}{0.3\textwidth}
        \centering
        \game{4,4}{2,2}{3,2}{\underline{1,3}}{16}
    \end{subfigure}
    \begin{subfigure}{0.3\textwidth}
        \centering
        \game{4,4}{3,1}{2,2}{\underline{1,3}}{17}
    \end{subfigure}
    \begin{subfigure}{0.3\textwidth}
        \centering
        \game{4,4}{2,1}{1,2}{\underline{3,3}}{18}
    \end{subfigure}
    \vspace{1em}

    % Row 7
    \begin{subfigure}{0.3\textwidth}
        \centering
        \game{5,5}{1,3}{3,1}{\underline{2,2}}{19}
    \end{subfigure}
    \begin{subfigure}{0.3\textwidth}
        \centering
        \game{5,5}{1,2}{3,1}{\underline{2,2}}{20}
    \end{subfigure}
    \begin{subfigure}{0.3\textwidth}
        \centering
        \game{5,5}{1,2}{2,1}{\underline{3,3}}{21}
    \end{subfigure}

    \caption{All structurally distinct $2 \times 2$ no-conflict matrix games from Section 11.2.1 of~\citet{albrecht2024multi}. Each cell shows the reward pair $(r_1, r_2)$ for Agents 1 and 2. We bold the welfare-optimal Nash equilibrium and underline any welfare-suboptimal Nash equilibria. 
    In all games, the optimal outcome is $(A,A)$. 
    To ensure the true policy gradient w.r.t.\@ uniformly random policies increases the probability of the optimal outcome, we change the reward associated with the optimal outcome to $(4,5)$ in games 7-12 and $(5,5)$ in games 19-21.}
    \label{fig:2x2_games_all}
\end{figure}

\textbf{BoulderPush:} 
In this game, agents must coordinate to push a boulder to a specified goal position.
Agents have four actions that move them one position up, down, left, or right.
To push the boulder, all agents must move to the positions above the boulder (\textit{i.e.}, the cells indicating the direction in which the boulder must be pushed) and then move in the indicated direction.
Agents receive reward $0.1$ for successfully pushing the boulder.
If only one agent attempts to push the boulder independently (without the other agent),  they receive reward $-0.02$.

\textbf{Level-based Foraging:}
In this game, agents must coordinate to collect food items on the game board.
Agents have five actions; four actions can move them one position up, down, left, or right. 
The fifth action is the ``forage'' action where the agent attempts to forage a food item.
To forage the food, both agents must move next to it and choose the forage action.
Agents receive reward $0.5$ for successfully foraging the food.
If only one agent attempts to forage independently (without the other agent),  they receive reward $-0.015$.

\subsection{Reward Modifications}

Our work focuses on improving the reliability of independent on-policy policy gradient algorithms \textit{when the expected policy gradient of each agent points toward optimality at initialization.}
%
% In most of the games we consider, the penalties for failing to cooperate are so large that the expected policy gradient points away from optimality. 
%
However, in most of these games, the expected policy gradient points away from optimality at initialization.\footnote{\citet{christianos2022pareto} show that on-policy policy gradient algorithms like MAPPO and MAA2C consistently converge sub-optimally in the same tasks we consider.}  
The probability of all agents cooperating is very small at initialization, so the expected return of any agent $i$ acting optimally (\textit{e.g.} pushing the boulder) has lower expected return than acting sub-optimally (\textit{e.g.} avoiding the boulder)~\citep{christianos2022pareto,zhao2023optimistic}.
To ensure the task setting aligns with our problem of interest, we rescale the reward function in some environments so that the expected policy gradient under uniformly initialized policies encourages cooperation. 
Without this adjustment, these games would not reflect the failure mode we aim to study and would be unsuitable for testing our hypotheses.

\begin{itemize}
    \item \textbf{LBF:} We reduce the penalty for failed cooperation from $-0.1$ to $-0.015$.
    \item \textbf{BoulderPush:} We reduce the penalty for failed cooperation from $-0.1$ to $-0.02$.
    \item \textbf{$3\times3$ Climbing game:} We show the original Climbing game rewards in Fig.~\ref{fig:climbing_og} and our modified Climbing game in Fig.~\ref{fig:climbing2}. 
    \item \textbf{$3\times3$ Penalty game:} We show the original Penalty game rewards in Fig.~\ref{fig:penalty_og} and our modified Penalty game in Fig.~\ref{fig:penalty2}. 
    \item \textbf{$2\times2$ no-conflict matrix games:} In all games, the optimal outcome originally has reward $4,4$. In games 7-12, we change this reward to $4,5$. In games 19-21, we change it to $5,5$.
\end{itemize}
% In brief, we reduce penalties for failed cooperation in BoulderPush, LBF, and the $3\times 3$ matrix games, and increase the reward for the optimal outcome in $2\times 2$ games 7–12 and 19–21.

% \section{Potential causes of sub-optimal convergence}

% \begin{enumerate}
%     \item The expected policy gradient of each agent points away from the optimal outcome.
%     \item The expected policy gradient of each agent points toward the optimal outcome, but sampling error causes the empirical estimate of the policy gradients to point away from it.
% \end{enumerate}

\section{Additional Experiments}
\label{app:experiments}

% \begin{wrapfigure}{R}{0.5\linewidth}
% \vspace{-1.5em}
% % \begin{figure}
%     \begin{subfigure}{\linewidth}
%         \centering
%         \includegraphics[width=\linewidth]{figures/2x2_sampling_error_joint.png}
%         \caption{Joint sampling error.}
%         \label{fig:2x2_19_se_joint}
%     \end{subfigure}
%     \begin{subfigure}{\linewidth}
%         \centering
%         \includegraphics[width=\linewidth]{figures/2x2_sampling_error_0.png}
%         \caption{Sampling error w.r.t.\@ Agent 1.}
%         \label{fig:2x2_19_se_0}
%     \end{subfigure}
%     \caption{Sampling error curves for $2\times2$ matrix game 19. Solid curves denote means over 100 seeds. Shaded regions denote 95\%  bootstrap confidence intervals.}
%     % \vspace{-1em}
% \end{wrapfigure}
% \end{figure}

In this section, we provide additional experiments excluded from main paper due to space constraints:
\begin{enumerate}
    % \item Sampling error curves for RL training in the Penalty game (Fig.~\ref{fig:penalty_se}).
    \item Training curves for all 21 distinct $2\times2$ no-conflict matrix games using MAPPO (Fig.~\ref{fig:2x2_mappo_all}) and IPPO (Fig.~\ref{fig:2x2_ippo_all})
    \item Sampling error curves for RL training in all 21 distinct $2\times2$ no-conflict matrix games (Fig.~\ref{fig:2x2_se_joint_all_ippo_maprops} and Fig.~\ref{fig:2x2_se_joint_all_ippo_props}).
    \item Training curves for BoulderPush and Level-based foraging tasks using IPPO (Fig.~\ref{fig:gw_bpush_lbf_ippo}).
    % \item Sampling error curves for BoulderPush and Level-based foraging tasks using IPPO (Fig.~\ref{fig:boush_lbf_ippo_joint_se}).
\end{enumerate}

% \begin{figure}
% \centering
% % \begin{subfigure}{0.7\linewidth}
%     % \includegraphics[width=\linewidth]{figures/sampling_error_joint_climbing.png}
%     % \caption{Climbing Game }
% % \end{subfigure}
% % \begin{subfigure}{0.7\linewidth}
%     \includegraphics[width=\linewidth]{figures/sampling_error_joint_penalty.png}
%     % \caption{Penalty Game}
% % \end{subfigure}
% % \begin{subfigure}{0.7\linewidth}
% %     \includegraphics[width=\linewidth]{figures/mappo/sampling_error_joint_gw.png}
% %     \caption{Grid World}
% % \end{subfigure}
% \caption{Mean joint sampling error during training in the Penalty Game. Solid curves denote means over 100 seeds. Shaded regions denote 95\%  bootstrap confidence intervals. }
% \label{fig:penalty_se}
% \end{figure}

\begin{figure}
    \centering
    \includegraphics[width=0.95\linewidth]{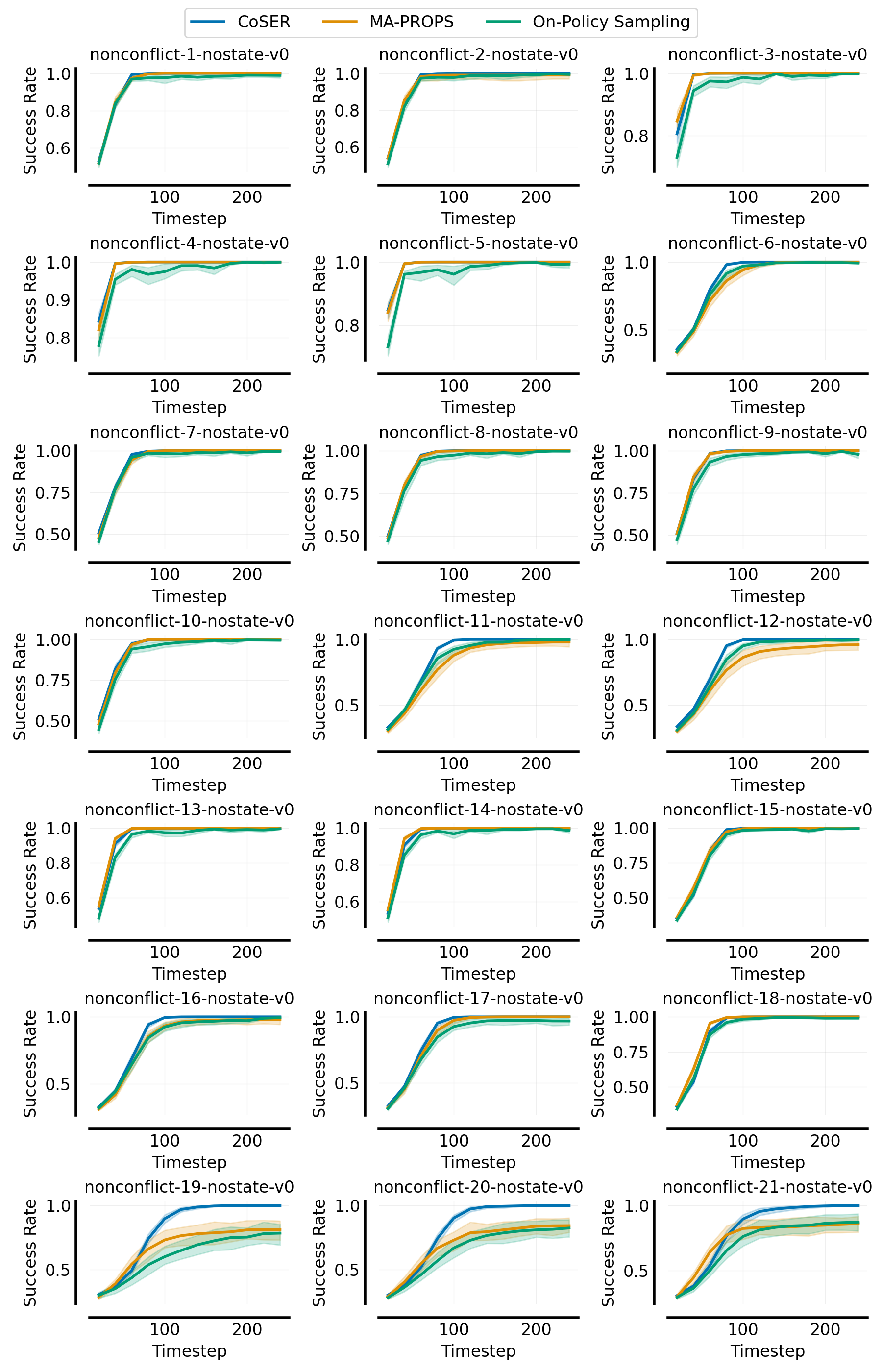}
    \caption{MAPPO mean success rate during training for all structurally distinct $2\times2$ no-conflict matrix games. Solid curves denote means over 100 seeds. Shaded regions denote 95\%  bootstrap confidence intervals. }
    \label{fig:2x2_mappo_all}
\end{figure}

\begin{figure}
    \centering
    \includegraphics[width=0.95\linewidth]{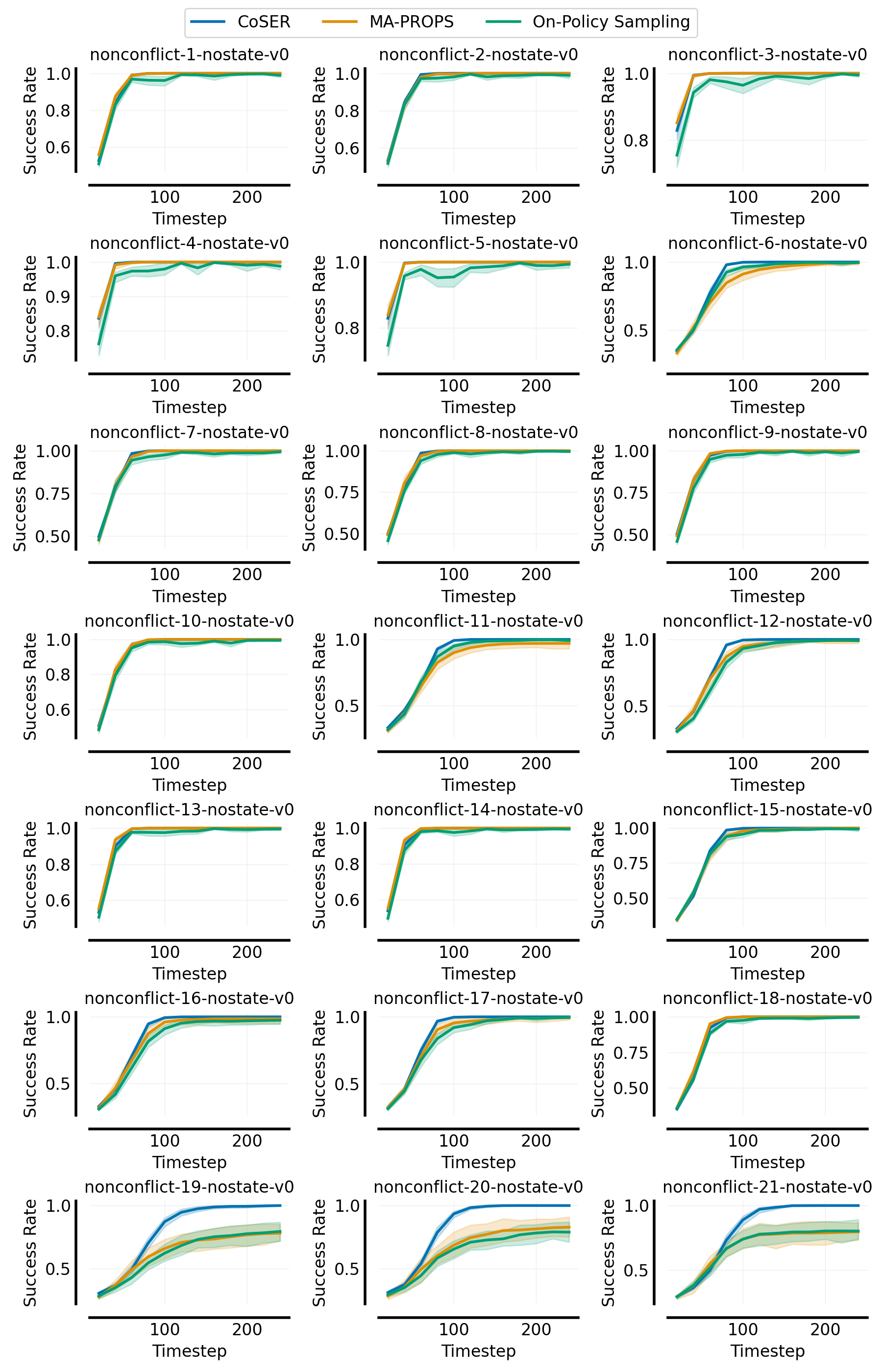}
    \caption{IPPO mean success rate during training for all structurally distinct $2\times2$ no-conflict matrix games. Solid curves denote means over 100 seeds. Shaded regions denote 95\%  bootstrap confidence intervals. }
    \label{fig:2x2_ippo_all}
\end{figure}

\begin{figure}
    \centering
    \includegraphics[width=0.95\linewidth]{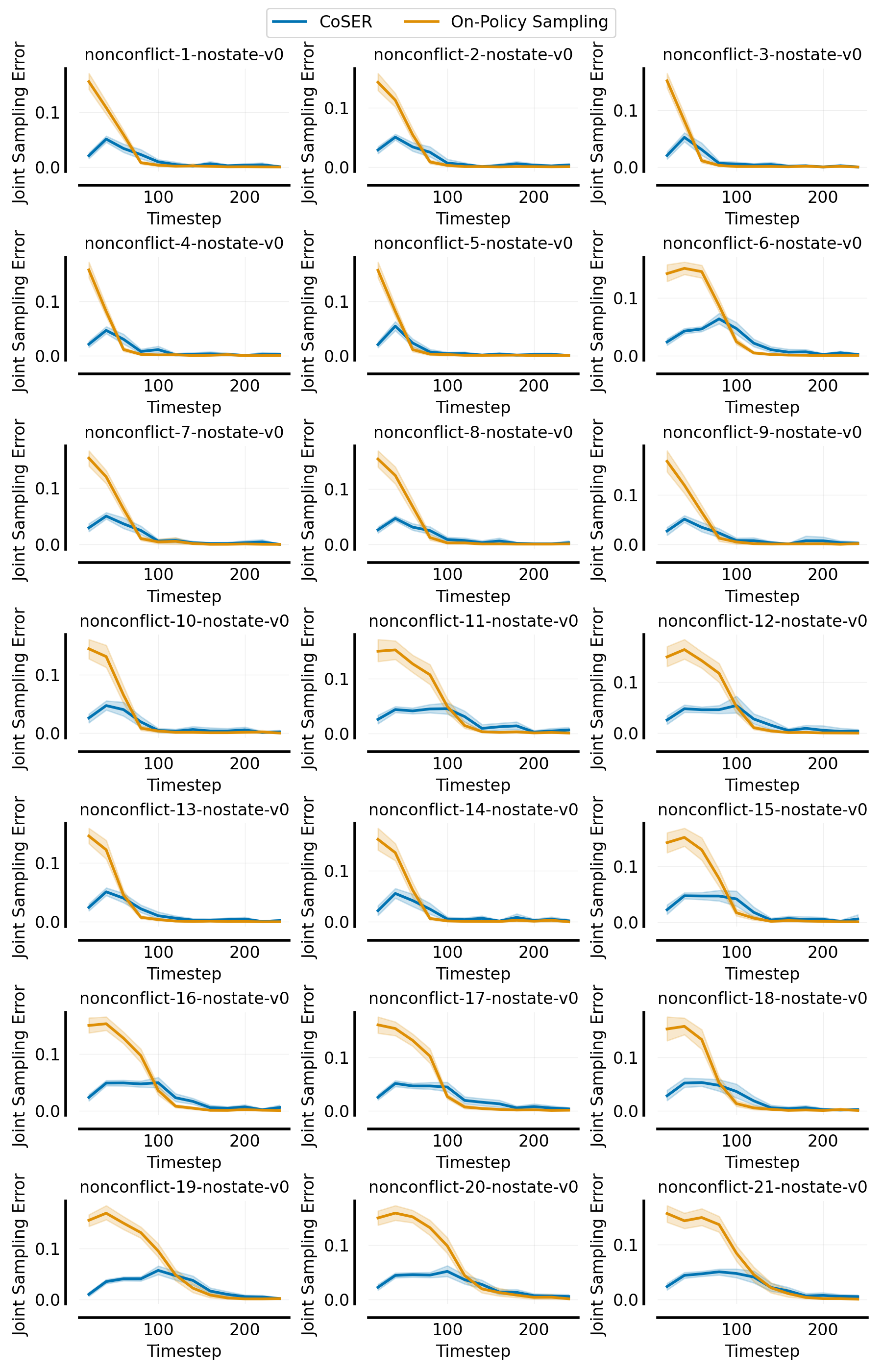}
    \caption{Mean joint sampling error for MAPPO + MA-PROPS training on all structurally distinct $2\times2$ no-conflict matrix games. Solid curves denote means over 100 seeds. Shaded regions denote 95\%  bootstrap confidence intervals.  {\color{MidnightBlue}\underline{Takeaway:} CoSER reduces sampling error w.r.t.\@ on-policy sampling by a large amount in all games.}}
    \label{fig:2x2_se_joint_all_ippo_maprops}
\end{figure}
\begin{figure}
    \centering
    \includegraphics[width=0.95\linewidth]{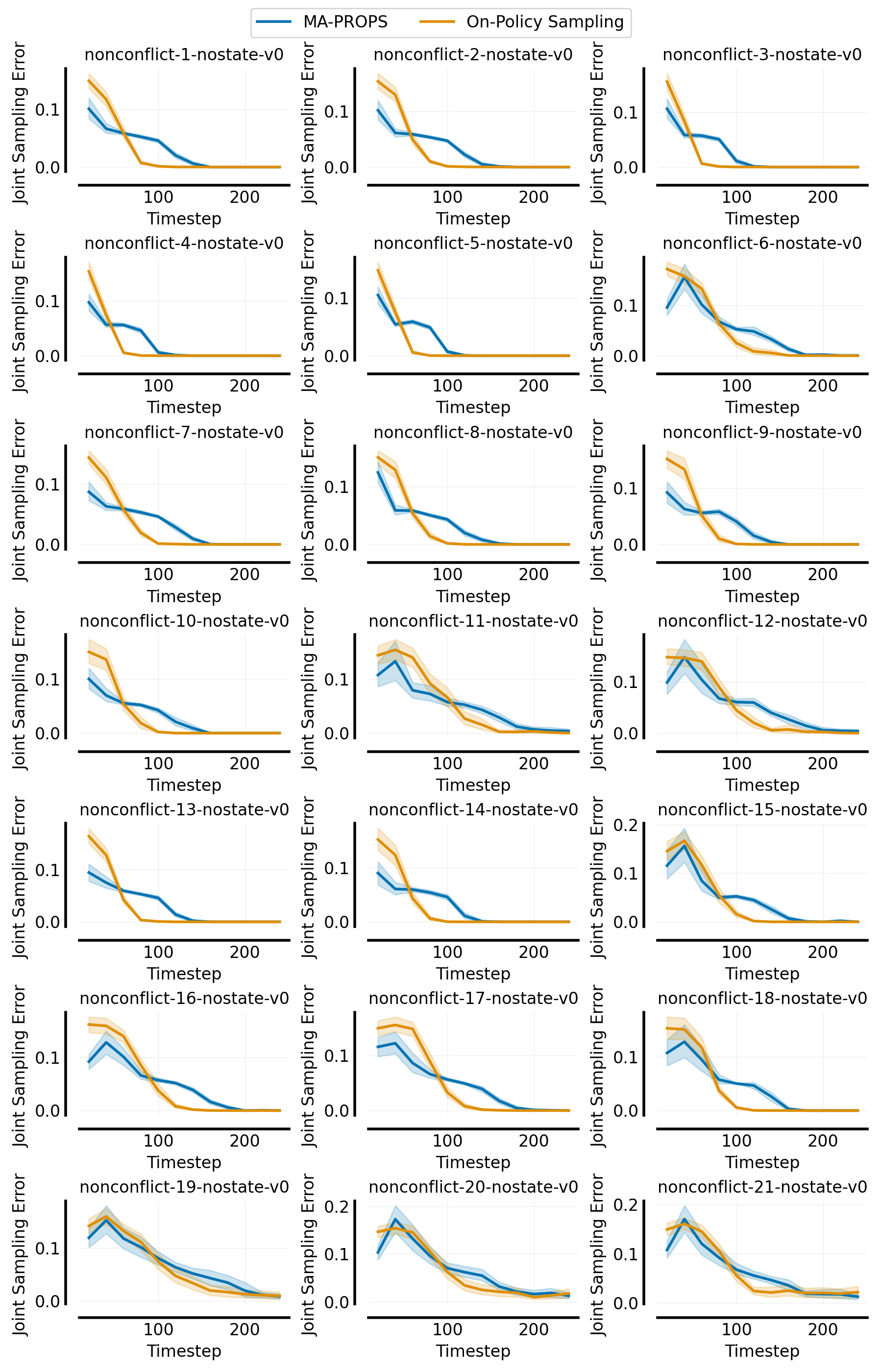}
    \caption{Mean joint sampling error for MAPPO + MA-PROPS training on all structurally distinct $2\times2$ no-conflict matrix games. Solid curves denote means over 100 seeds. Shaded regions denote 95\%  bootstrap confidence intervals. {\color{MidnightBlue}\underline{Takeaway:} Early in training, MA-PROPS reduces sampling error w.r.t.\@ on-policy sampling in all games except 19-21. CoSER reduces sampling error by a larger amount in all games (Fig.~\ref{fig:2x2_se_joint_all_ippo_maprops}).}}
    \label{fig:2x2_se_joint_all_ippo_props}
\end{figure}

\begin{figure}
    \centering
    \includegraphics[width=\linewidth]{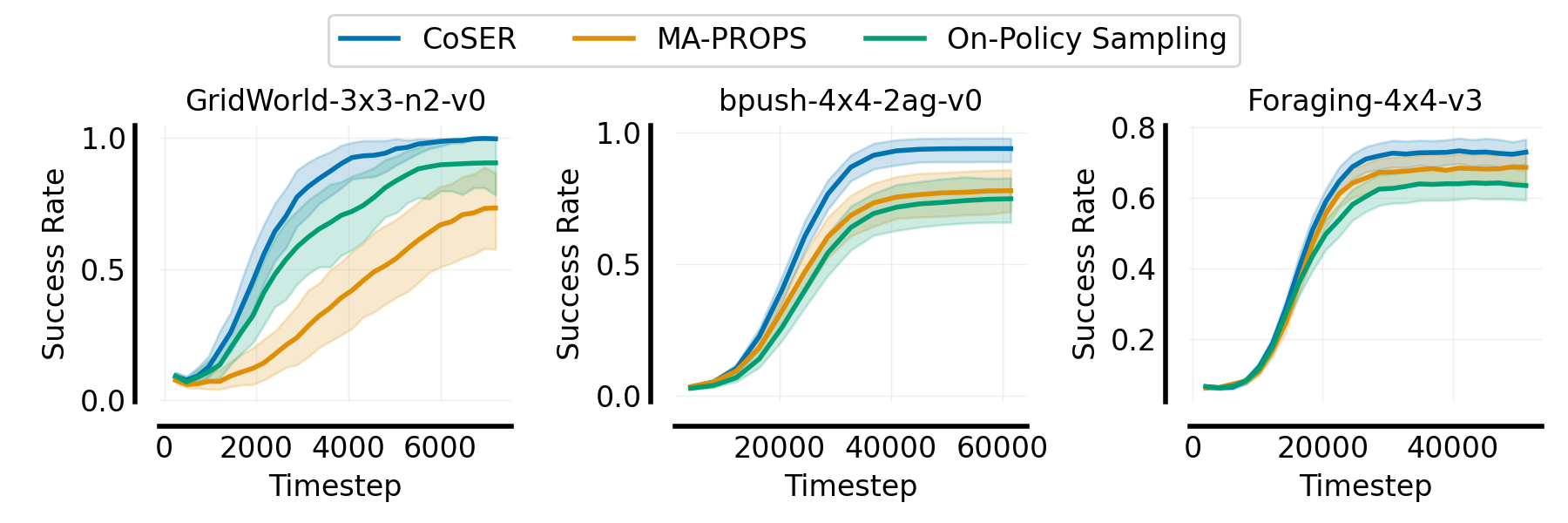}
    \caption{IPPO mean success rate during training for GridWorld, BoulderPush, and LBF. Solid curves denote means over 100 seeds. Shaded regions denote 95\%  bootstrap confidence intervals.}
    \label{fig:gw_bpush_lbf_ippo}
\end{figure}

\clearpage
\section{Hyperparameters}

All hyperparameters used in our experiments are listed in Tables~\ref{tab:se_updating_hyperparams_sweep}, ~\ref{tab:se_updating_hyperparams_fixed}, and~\ref{tab:hyperparameters_tuned}.
Our tuning procedure required running approximately one thousand jobs and requires access to many CPUs to do efficiently.
We ran all experiments on a computing cluster that enabled us to run several hundred jobs in parallel. The training budget for all experiments is fairly small (less than 100k timesteps), and CoSER and MA-PROPS jobs finish within 1-2 hours. Each job requires 0.6 GB of memory and 1.6 GB of disk.

\label{app:hyperparameters}

\begin{table*}[]
    \centering
    \begin{tabular}{l|l}
        \hline
        Behavior learning rate & $0.3, 0.03, 0.003$\\
        Behavior update frequency $m$ & $1, 4$ \\
        Behavior KL cutoff & $6$ \\
        Behavior clipping coefficient & $0.3, 1, 10$ \\
        \hline
    \end{tabular}
    \caption{Hyperparameters used in our hyperparameter sweep for training.}
    \label{tab:se_updating_hyperparams_sweep}
\end{table*}

\label{app:training}
\begin{table*}[]
    \centering
    \begin{tabular}{l|l}
        \hline
        Batch size & See Table.~\ref{tab:hyperparameters_tuned} \\
        Learning rate & See Table.~\ref{tab:hyperparameters_tuned}\\
        Number of update epochs & $4$ \\
        % \MA-PROPS{} number of update epochs & 4 \\
        % \hline
        Number of minibatch updates per epoch & $4$ \\
        % \MA-PROPS{} minibatch size for \MA-PROPS{} update & $bn/16$ \\
        Discount factor $\gamma$ & $0.99$ \\
        generalized advantage estimation (GAE) & $0.95$ \\
        Advantage normalization & Yes \\
        Loss clip coefficient & $0.2$ \\
        Entropy coefficient & $0.01$ \\
        Value function coefficient & $0.5$ \\
        Gradient clipping (max gradient norm) & $0.5$ \\
        KL cut-off & None \\
        \hline
        Actor and critic network architectures & Multi-layer perceptron \\
        & with hidden layers $(64,64)$ \\
        \hline
        Optimizer & Adam~\citep{kigma2015adam} \\
        \hline
        Number of evaluation episodes & $100$ \\
        \hline
    \end{tabular}
    \caption{MAPPO/IPPO hyperparameters across all experiments. We implement MAPPO/IPPO on top of the \PPO{} implementation provided by CleanRL~\citep{huang2022cleanrl}.}
    \label{tab:se_updating_hyperparams_fixed}
\end{table*}

\begin{table*}[]
\centering
\begin{tabular}{l|c|c|c|c}
& PPO & PPO & CoSER/MA-PROPS & CoSER/MA-PROPS\\
Game & Batch Size & Learning Rate & Learning Rate & Update Frequency \\
\hline
LBF                     & 2048  & $0.01$    & $0.03$     & 1 \\
BoulderPush             & 4096  & $0.003$   & $0.03$     & 1 \\
GridWorld               & 256   & $0.01$     & $0.3$ & 1    \\
$3\times3$ matrix games & 45    & $0.1$     & $0.3$  & 1    \\
$2\times2$ matrix games & 20    & $0.1$     & $0.03$  & 1   \\
\hline
\end{tabular}
\caption{Tuned hyperparameters used in \RL{} training with CoSER and MA-PROPS.}
\label{tab:hyperparameters_tuned}
\end{table*}

\end{document}